\documentclass{article}

\usepackage[abbrvbib, preprint]{jmlr2e}
\setcitestyle{numbers,square}
\bibpunct{[}{]}{,}{n}{}{,}
\hypersetup{
  colorlinks=true,
  citecolor=blue
}

\usepackage{multirow}
\usepackage{algorithm}
\usepackage{algorithmicx}
\usepackage{algpseudocode}
\usepackage[table]{xcolor}
\usepackage{soul}
\usepackage{caption}
\usepackage{graphicx}
\usepackage{xspace}
\usepackage{amssymb}

\definecolor{my_purple}{HTML}{9903F0}
\definecolor{my_green}{HTML}{156C09}
\definecolor{my_orange}{HTML}{FF3300}

\definecolor{light-gray}{gray}{0.8}
\definecolor{light-green}{RGB}{175, 227, 183}
\definecolor{light-yellow}{RGB}{227, 226, 175}
\definecolor{light-blue}{RGB}{175, 227, 221}

\DeclareRobustCommand{\hlgrey}[1]{{\sethlcolor{light-gray}\hl{#1}}}
\DeclareRobustCommand{\hlgreen}[1]{{\sethlcolor{light-green}\hl{#1}}}
\DeclareRobustCommand{\hlyellow}[1]{{\sethlcolor{light-yellow}\hl{#1}}}
\DeclareRobustCommand{\hlblue}[1]{{\sethlcolor{light-blue}\hl{#1}}}

\makeatletter
\DeclareRobustCommand{\onedot}{%
  \futurelet\@let@token\@onedot
}
\def\@onedot{%
  \ifx\@let@token.%
  \else.%
  \null%
  \fi%
  \xspace%
}
\makeatother
\newcommand{\eg}{\emph{e.g}\onedot}   
\newcommand{\ie}{\emph{i.e}\onedot}

\newcommand{\etal}{\emph{et al}\onedot}
\definecolor{sotaColor}{RGB}{153, 0, 102} 
\newcommand{\sota}[1]{\textcolor{sotaColor}{#1}}

\newcommand{\bluecol}[1]{\textcolor{blue}{#1}}
\newcommand{\greencol}[1]{\textcolor{my_green}{#1}}

\usepackage{bm}
\usepackage{amsmath}

\newcommand{\Gr}[2]{\mathrm{Gr}(#1, #2)}

\def\1{\bm{1}}

\def\va{{\bm{a}}}
\def\vb{{\bm{b}}}

\def\vh{{\bm{h}}}

\def\vo{{\bm{o}}}

\def\vx{{\bm{x}}}

\def\mA{{\mathbf{A}}}
\def\mB{{\mathbf{B}}}
\def\mC{{\mathbf{C}}}

\def\mG{{\mathbf{G}}}

\def\mI{{\mathbf{I}}}

\def\mO{{\mathbf{O}}}
\def\mP{{\mathbf{P}}}
\def\mQ{{\mathbf{Q}}}

\def\mU{{\mathbf{U}}}
\def\mV{{\mathbf{V}}}
\def\mW{{\mathbf{W}}}
\def\mX{{\mathbf{X}}}

\def\mZ{{\mathbf{Z}}}

\DeclareMathAlphabet{\mathsfit}{\encodingdefault}{\sfdefault}{m}{sl}
\SetMathAlphabet{\mathsfit}{bold}{\encodingdefault}{\sfdefault}{bx}{n}

\newcommand{\E}{\mathbb{E}}
\newcommand{\R}{\mathbb{R}}

\def\vomega{{\bm{\omega}}}

\DeclareMathOperator{\Tr}{Tr}

\newcommand{\Loss}[1]{\text{L}_{\texttt{#1}}}

\newcommand{\GL}[1]{\mathrm{GL}(#1)}

\newcommand\defeq{\mathrel{\stackrel{\makebox[0pt]{\mbox{\normalfont\tiny def}}}{=}}}

\newcommand{\std}[1]{\textsubscript{\textcolor{gray}{\tiny$\pm$#1}}}

\usepackage[utf8]{inputenc} 
\usepackage[T1]{fontenc}    
\usepackage{hyperref}       
\usepackage{url}            
\usepackage{booktabs}       
\usepackage{amsfonts}       
\usepackage{nicefrac}       
\usepackage{microtype}      
\usepackage{xcolor}         

\title{Exemplar-Free Continual Learning for State Space Models}

\author{
\makebox[\textwidth][c]{%
\parbox{0.9\textwidth}{\centering
{\bfseries
Isaac Ning Lee$^{1,2}$ \quad
Leila Mahmoodi$^{1}$ \quad
Trung Le$^{1}$ \quad
Mehrtash Harandi$^{1}$}
\vspace{0.4em}
{\itshape
$^{1}$Monash University \qquad $^{2}$National University of Singapore}
}}
}

\begin{document}
\maketitle

\begin{abstract}
State-Space Models (SSMs) excel at capturing long-range dependencies with structured recurrence, making them well-suited for sequence modeling. However, their evolving internal states pose unique challenges in Continual Learning (CL). Without access to the full distribution of previous tasks, updates to the state-space dynamics become unconstrained, leading to catastrophic forgetting. To address this, we propose \textbf{Inf-SSM}, a geometry-aware regularization framework for CL in SSMs. It constrains state evolution via the infinite-dimensional Grassmannian of SSM observability subspaces, without requiring any exemplars from past tasks. Unlike classical CL methods that restrict weight updates, Inf-SSM directly regularizes the infinite-horizon state evolution encoded by the extended observability subspace of the SSM. We show that enforcing this regularization requires solving a matrix equation known as the Sylvester equation, which typically incurs $\mathcal{O}(n^3)$ complexity. Thus, we develop an $\mathcal{O}(n^2)$ solution by exploiting the structure and properties of SSMs. This leads to an efficient regularization mechanism that can be seamlessly integrated into existing CL methods. Comprehensive experiments on challenging benchmarks of ImageNet-R, CIFAR-100, and Caltech-256 demonstrate a significant reduction in forgetting while improving accuracy across sequential tasks.
\end{abstract}   
\newpage
\section{Introduction}
\label{sec:intro}
In this paper, we propose \textbf{Inf-SSM}, a novel regularization method to equip State-Space Models (SSMs)~\cite{gu2022efficientlyS4} with Continual Learning (CL) capabilities, enabling them to integrate new information while preserving previously learned knowledge without any past exemplars.

State-space models (SSMs)~\cite{gu2022efficientlyS4} have emerged as structured sequence models that offer a scalable and efficient alternative to transformers. Recent architectures such as S4, S6, and Mamba-2~\cite{gu2022efficientlyS4, gu2023mamba, dao2024transformers} effectively capture long-range dependencies with linear computational complexity, achieving strong performance across NLP, Vision, Generative Models, Acoustics, and Robotics~\cite{gu2023mamba,zhu2024VIM,phung2024dimsum, zhang2024motion,liang2024self,liu2024robomamba}.
In NLP, Mamba~\cite{gu2023mamba} has positioned SSMs as competitive alternatives to transformers. In vision, \underline{\textbf{Vi}}sion \underline{\textbf{M}}amba (Vim)~\cite{zhu2024VIM} and its variants~\cite{liu2025vmamba, hatamizadeh2024mambavision} have demonstrated strong performance in spatial-temporal modeling while offering faster inference than attention-based architectures.

Beyond their performance, SSMs provide intrinsic advantages such as structured recurrence and linear memory complexity, making them well-suited for sequence modeling. However, \textbf{a major gap remains}: adapting SSMs to continual learning while mitigating Catastrophic Forgetting (CF)~\cite{mccloskey1989catastrophic, mcclelland1995there} remains largely \textbf{unexplored}~\cite{2024WangSurvey}. 
CF occurs when a model loses previously learned knowledge as it is trained on new tasks. Despite their promising capabilities, the structured recurrence of SSMs makes continual and lifelong adaptation non-trivial, demanding specialized techniques to avoid CF~\cite{gu2023mamba, zhu2024VIM}.

While well-established CL algorithms designed for MLPs, CNNs, and Transformers can be directly adapted to SSMs and yield competitive baselines, such adaptations typically treat SSM parameters as generic weights and overlook the underlying state-space geometry and temporal dynamics. As a result, they fail to fully exploit the structural advantages of SSMs in continual learning. This motivates the development of CL methods that are derived from the true state-space representation and regularize the model in a way that is faithful to its geometry.


\paragraph{Our Contribution.}
A key property of SSMs, as we will show shortly, is their invariance in system representations~\cite{Afsari2013LDSAlign, ravichandran2012categorizing}. This allows us to describe an SSM model using its extended observability subspace~\cite{de2002subspace, ravichandran2012categorizing}. The extended observability subspace provides a structured representation of an SSM, capturing its complete behavior through an infinite unrolling of the system’s response.

To work with extended observability subspaces, we adopt the geometry of the Grassmann manifold~\cite{huang2016sparse}, which is the natural choice for analyzing subspaces. Such modeling captures the system representations' invariance by identifying all orthonormal bases spanning the same subspace as a single point, and endows the space with a smooth manifold structure~\cite{de2002subspace, turaga2011statistical, Afsari2013LDSAlign}. However, since the extended observability subspace belongs to an infinite-dimensional Grassmannian, direct distance computation is far from trivial. We address this challenge by demonstrating that, for SSMs, computing distances on the infinite Grassmannian is tractable, significantly reducing the computational complexity of our algorithm. Solving the computation challenges allows us to leverage the extended observability subspace to retain the model's past knowledge effectively \textbf{without the need for storing past examples.} This enables Inf-SSM to be complementary to existing CL techniques, regardless of the availability of past task data or stored feature embeddings.

As a preview, Inf-SSM substantially boosts the performance of existing CL methods when integrated with them and consistently outperforms memory-less baselines in both accuracy and efficiency, as illustrated in \cref{fig:teaser}.
%

\noindent In summary:
\begin{itemize}
    \item We proposed Inf-SSM, a simple state regularization loss for CL in SSMs based on the extended observability subspace. Inf-SSM is an exemplar-free method and does not rely on stored information from prior tasks to function.
    \item We demonstrate that Inf-SSM is computationally efficient by exploiting structural properties in SSMs' states to reduce the computational complexity 
    from $\boldsymbol{\mathcal{O}(n^3)}$ to $\boldsymbol{\mathcal{O}(n^2)}$ and reduce the FLOPS count up to $\boldsymbol{100\times}$. The simple and exemplar-free design of Inf-SSM makes it plug-and-play to improve existing CL algorithms, regardless of whether they are replay or replay-free methods.
    \item We empirically validate Inf-SSM across diverse CL paradigms, where it reduces $\mathrm{FM}$ by \textbf{9.36\%} and increases $\mathrm{AA}$ by \textbf{8.31\%} on average across the 5 and 10 task settings on three datasets when integrated with existing CL baselines.
\end{itemize}

\section{Preliminary}
\label{sec:preliminary}
In this section, we will first introduce the notation used throughout the paper. Then, we will formalize the Class-incremental Learning (CIL) and Exemplar-free Class-incremental Learning (EFCIL)
problem in \textsection\ref{sec:problem_statement}. Finally, we will review the theory of Grassmannian and SSM in \textsection\ref{sec:grassmannian} and \textsection\ref{sec:ssm}, which underpin our proposed method.

\paragraph{Notations.}
In this paper, matrices are denoted as bold capital letters $\mX$ and column vectors as bold lower-case letters $\vx$. The $n \times n$ identity matrix is shown by $\mI_n$. The diagonal elements of $\mX \in \R^{n \times n}$ is shown by $\mX_{diag}$. 
The \emph{Frobenius norm} of \(\mathbf{A} \in \R^{m \times n}\) is given by $\| \mA \|_{\text{F}} = \sqrt{\sum_{i=1}^m\sum_{j=1}^n\mA[i,j]^2}$ and the \emph{Hadamard product} of two matrices $A,B\in\mathbb{R}^{m\times n}$ is denoted $A\odot B$ and defined entrywise by $(A\odot B)_{ij}=A_{ij}B_{ij}$. 
We use $\vx[\cdot]$ to denote discrete-time vectors and $\vx(\cdot)$ for continuous-time vectors. We write $\mathrm{SN}(x) = 2/(1+\exp{(-x)}) - 1$ as the Soft-normalization function~\cite{huang2017efficient}. 
The general linear group $\GL{n}$ is the set of all $n \times n$ invertible matrices and is formally defined as: $\GL{n} = \{ \mP \in \R^{n \times n} \mid \det(\mP) \neq 0 \}$. See \textsection\ref{app:notation} for additional notation. 

\subsection{Problem Statement}
\label{sec:problem_statement}

Let us define a non-stationary task sequence $\{\mathcal{T}_T\}_{t=1}^{N}$, where $T$ denotes the task identifier, and $N$ is the total number of tasks. The data distribution of task~$T$ is denoted by $\mathcal{D}_{T}=\{\mathcal{X}_{T}, \mathcal{Y}_{T}\}$; where $\mathcal{X}_{T}$ and $\mathcal{Y}_{T}$ are input space and class labels, specific to task~$T$. The number of classes in task~$T$ is shown by $\mathcal{C}_T$, and the label sets of every two tasks are disjoint, meaning $\mathcal{Y}_i \cap \mathcal{Y}_j = \emptyset$ for $i\neq j$.

In the CIL protocol, the task identifier $T$ is not available at inference time. The EFCIL protocol further tightens this setting: during training on task $T$, the model only has access to the current data distribution $\mathcal{D}_T$, and no samples or feature embeddings from previous distributions $\{\mathcal{D}_i\}_{i=1}^{T-1}$ are allowed. This exemplar-free constraint makes EFCIL more challenging but also more flexible, as it removes the need to store data from past tasks.


\subsection{Grassmann Manifold}
\label{sec:grassmannian}

The \emph{Grassmann manifold} $\Gr{n}{d}$ is the set of all $n$-dimensional linear subspaces of $\R^d$. Formally, it is defined as
\begin{align}
    \Gr{n}{d} = \bigl\{\mX \in \mathbb{R}^{d \times n} \mid \mX^\top\mX = \mI_n \bigr\}/O(n)\;,  
    \label{eqn:grassmann_definition}
\end{align}
where $O(n)$ denotes the orthogonal group of $n \times n$ matrices, accounting for the fact that different orthonormal bases can represent the same subspace. The \textbf{infinite Grassmannian} is defined as
\begin{align}
    \Gr{n}{\infty} = \lim_{d \to \infty} \Gr{n}{d}, 
\end{align}
which encapsulates all possible $n$-dimensional linear subspaces in an infinite-dimensional Hilbert space~\cite{ye2016schubert}.



\subsection{State-Space Models}
\label{sec:ssm}



Recent structured SSM architectures~\cite{gu2022efficientlyS4,gu2022parameterizationS4D,gu2023mamba,zhu2024VIM,dao2024transformers} are built upon the discretized state-space model\footnote{Details on discretization are provided in \textsection\ref{app:discrete_SSM}.}:
\begin{equation}
\begin{aligned}
    \vh[t] &= \mA \vh[t-1] + \mB x[t], \\
    y[t] &= \mC \vh[t],
\end{aligned}
\label{Eq:SSM_disc}
\end{equation}
where $\vh[t] \in \R^n$ is the hidden state, and $x[t], y[t] \in \R$ denote the input and output, respectively. The model is parameterized by the state transition matrix $\mA \in \R^{n \times n}$, input matrix $\mB \in \R^{n \times 1}$, and output matrix $\mC \in \R^{1 \times n}$. Classical formulations assume $(\mA,\mB,\mC)$ to be linear time-invariant (LTI), while more recent architectures such as Mamba~\cite{gu2023mamba}, Mamba-2~\cite{dao2024transformers}, and Vim~\cite{zhu2024VIM} introduce input- or position-dependent dynamics, effectively yielding time-varying state transitions. Such a dynamical nature allows a more adaptive model, but at the same time, poses a significant problem in regularizing the model's state.

Moreover, earlier SSMs, such as S4~\cite{gu2022efficientlyS4} use a full $\mA \in \R^{n \times n}$, whereas many recent models~\cite{gupta2022diagonal,gu2022parameterizationS4D,gu2023mamba,zhu2024VIM,liu2025vmamba} exploit diagonal parameterizations of $\mA$ to improve efficiency while preserving expressivity. Further details on SSMs, Mamba, Mamba-2, and Vim are provided in \textsection\ref{app:inf-ssm}.

\begin{figure}[t]
    \centering
    \begin{minipage}[t]{0.49\linewidth}
        \centering
        \includegraphics[width=\linewidth]{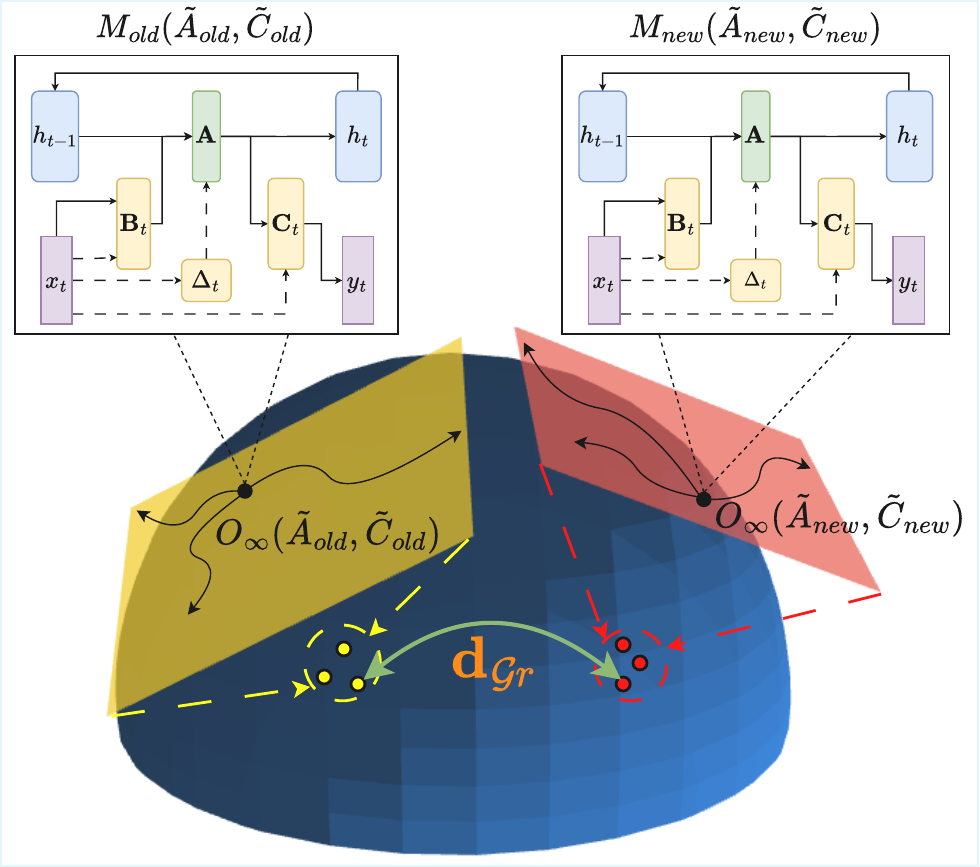}
        \caption{SSM at each sequence position $\tau$ is characterized by the infinite-horizon observability subspace $O_{\infty}$ defined by the tuple $(\tilde{\mA}, \tilde{\mC})$, and visualized as a trajectory to an infinite horizon. The colored plane represents the complete set of $O_{\infty}$. Each trajectory is mapped to a point on the Grassmannian, and the pairwise distance $d_{\mathcal{G}r}$ is illustrated as the geodesic on the sphere representing the Grassmann manifold.}
        \label{fig:inf-grass}
    \end{minipage}
    \hfill
    \begin{minipage}[t]{0.49\linewidth}
        \centering
        \includegraphics[width=\linewidth]{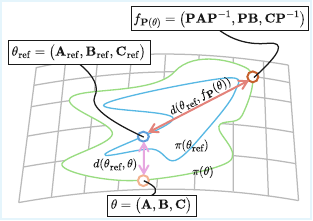}
        \caption{\greencol{Green}: Orbit $\pi(\theta)$ of the target SSM $\theta=(\mA,\mB,\mC)$; \bluecol{Blue}: Orbit $\pi(\theta_{\mathrm{ref}})$ of the reference. A state reparameterization $x\mapsto \mP x$ moves $\theta$ to $f_{\mP}(\theta)=(\mP\mA\mP^{-1}, \mP\mB, \mC\mP^{-1})$, preserving SSM's behavior but changing the Frobenius distance to the reference. Thus, the Frobenius norm is not invariant to P-equivalence.}
        \label{fig:p-equiv}
    \end{minipage}
\end{figure}

\newpage

\section{Proposed Method}
\label{sec:proposed_method}

Our aim is to endow SSMs and their variants with the ability to preserve their prior knowledge while gaining new. Consider a network trained on tasks $1{:}T{-}1$ with parameters $\vomega_{T-1}$. Continual learning on the incoming dataset \(\mathcal{D}_T\) typically follows two strategies: \textbf{(1)} Regularize the parameters for the incoming new data $\mathcal{D}_T$, making sure that the model does not significantly diverge from its initial parameters $\vomega_{T-1}$; \textbf{(2)} Regularize the output behavior of the model to ensure the output space of the new model encapsulates the previous tasks' knowledge. 

Translating such an idea to an SSM implies regularizing $\mA, \mB$ and $\mC$ or directly constraining the output $y_t$ of the model. We discuss both possibilities in more detail below.




\paragraph{Regularizing the parameters.} 
One can regularize the parameters of the SSM, \ie $(\mA,\mB,\mC)$, leading to 
\begin{align}
    \text{L}\big( \mA_T, \mB_T,\mC_T \big) &= \big\| \mA_T - \mA_{T-1} \big\|_\text{F}^2\ 
    + \big\| \mB_T - \mB_{T-1} \big\|_\text{F}^2\ \notag \notag \\&+  \big\| \mC_T - \mC_{T-1} \big\|_\text{F}^2\;.
    \label{eqn:reg_ABC}
\end{align}

The problem here is far more significant. The normal regularization of parameters, while implicitly addressing the finite horizon problem, is oblivious to the geometry of SSMs and totally ignores the P-equivalence. 

\begin{definition}[P-equivalence]
Let an SSM be represented by the tuple of system parameters $(\mA, \mB, \mC)$. For an invertible matrix $\mP \in \GL{n}$, two SSMs with system parameters $(\mA, \mB, \mC) $ and $(\mP \mA \mP^{-1}, \mP \mB, \mC \mP^{-1})$ are said to be \textbf{P-equivalent} since they represent the same system behavior. That is, the transformation: 
$(\mA' = \mP \mA \mP^{-1}, \mB' = \mP \mB, \mC' = \mC \mP^{-1})$ preserves the input-output mapping (\ie, the two systems are indistinguishable in terms of their external behavior).  
\end{definition}

This has a significant implication as one can move along the orbit defined by the P-equivalence without affecting the system behavior. In other words, there are infinitely many parameter realizations for an SSM that lead to exactly the same sample path $y[t]$. As such and to regularize an SSM, the loss needs to be invariant to $(\mA, \mB, \mC) \to (\mP\mA\mP^{-1}, \mP\mB, \mC\mP^{-1})$ for any $\mP \in \GL{n}$, as illustrated in \cref{fig:p-equiv}. We will show that our algorithm leads to a loss that is faithful to P-equivalence.

\paragraph{Regularizing the output.} 
Alternatively, one can also regularize the output behavior of the model. This principle is the basis of the replay-based methods, which maintain a memory $\mathcal{M}$ containing representative samples from past experience. The objective typically enforces consistency between the current model and its previous version over these stored samples:
\begin{align}
    \text{L}\big( \mA_T, \mB_T, \mC_T \big) 
    = \underset{\vx \sim \mathcal{M}}{\mathbb{E}} \sum_{t=1}^{\tau} \big\| y_T[t] - y_{T-1}[t] \big\|^2 \;,
    \label{eqn:finite_horizon_issue}
\end{align}
with $\tau$ denoting the time horizon. Ideally, one would take $\tau \to \infty$ so that the entire temporal behavior of the model is preserved as, in principle, SSMs can induce an output with an infinite sequence even with a finite input.

As in the EFCIL setting, storing $\mathcal{M}$ is not feasible, thus, we view the model as a dynamical system and consider its response when excited by random Gaussian input. Such excitation statistically probes the system and characterizes its behavior without requiring any stored samples. We will show soon that this idea allows our algorithm to operate effectively in the EFCIL regime, achieving the benefits of replay without memory while capturing the model’s infinite-horizon behavior.


\paragraph{Our method.} 
Although vanilla regularization (\eg, \Cref{eqn:reg_ABC}) improves CL performance which we will show in \textsection\ref{sec:results}, SSMs still suffer from significant forgetting. We argue that a key reason is that standard penalties \textbf{ignore the geometry of SSMs.}
In particular, moving along the orbit
\begin{equation}
    \pi = \{(\mP\mA\mP^{-1},\mP\mB,\mC\mP^{-1}), \forall \mP \in \GL{n}\}\;,
    \label{eqn:pi_orbit}
\end{equation}
does not change the input-output behavior of the SSM, yet conventional regularizers are not invariant to this transformation. As a result, they may over-penalize parameter updates that are functionally equivalent and under-penalize updates that alter the underlying dynamics. 

What we propose to do instead is to take into account the trajectory defined by an SSM and use that to preserve the knowledge. This perspective is naturally compatible with arbitrarily long horizons as a byproduct of our construction. We will discuss this below, but first, we need to take a detour and introduce the Extended observability matrix of an SSM.

%


\subsection{Extended Observability}
\label{sec:extende_observability}

Consider studying the behavior for an SSM excited by $x[t] \sim \mathcal{N}(0,1)$. 
The expected response of the system can be derived as 
\begin{align}    
    \E \big[y[t]\big] &= \E \big[ \mC \vh[t] \big] =   \E \big[ \mC \big( \mA \vh[t-1] + \mB x[t]\big)  \big] 
    \notag \\    &
    = \mC \mA \E \big[ \vh[t-1] \big]\;.     
    \label{eqn:expected_y}
\end{align}
%

This implies,
%
%
\begin{align}
    \begin{bmatrix}
        \E\big[y[1]\big]\\
        \E\big[y[2]\big]\\
        \E\big[y[3]\big]\\
        \vdots
    \end{bmatrix}
    = 
    \begin{bmatrix}
        \mC\\
        \mC\mA\\
        \mC\mA^2\\
        \vdots
    \end{bmatrix}
    \vh[0]    
\end{align}

In other words, the dynamics of an SSM can be well-captured by the pair $(\mA,\mC)$. The extended observability matrix is defined accordingly.

\begin{definition}[Extended Observability]
The extended observability of an SSM with parameters $(\mA,\mB,\mC)$ is defined as:
%
\begin{align}
    \label{eqn:extended_obs}
    \mO_{\infty}(\mA, \mC) \defeq \begin{bmatrix}
        \mC \\
        \mC \mA \\
        \mC \mA^2 \\
        \vdots
    \end{bmatrix} \in \R^{\infty \times n}.
\end{align}
\end{definition}

One important emerging property of the observability matrix for an SSM is that the subspace spanned by the columns of $\mO_{\infty}(\mA, \mC)$ is invariant to the orbit defined by \cref{eqn:pi_orbit}. More specifically, let $\mathcal{S}_{\infty}(\mA,\mC)$ be the subspace spanned by the columns of $\mO_{\infty}(\mA,\mC)$. Then we have the following essential result.

\begin{theorem}[Invariance of the $\mathcal{S}_{\infty}$ under P-equivalence]
\label{thm:obs_equiv_main}
Let $(\mA,\mB,\mC)$ and $(\mA'=\mP\mA\mP^{-1},\mB'=\mP\mB,\mC'=\mC\mP^{-1})$ be two equivalent representations for an SSM for $\mP \in \GL{n}$. The subspace spanned by the extended observability matrices remains unchanged. That is
\begin{align}
    \mathcal{S}_{\infty}(\mA', \mC') = \mathcal{S}_{\infty}(\mA, \mC)\;.
\end{align}
\end{theorem}
This is a well-known result in the context of system identification in linear dynamical systems (see \cite{huang2017efficient}). We also provide a proof in \cref{thm:Obs_equiv} in \textsection\ref{app:obs_equiv} to have a self-contained work. 

We note that $\mathcal{S}_{\infty}(\mA,\mC) \in \Gr{n}{\infty}$. 
Although with $\mathcal{S}_{\infty}(\mA,\mC)$, we can describe the expected output of an SSM uniquely, to use it for regularizing an SSM and deploying it for CL, we need one more step. That is, given two subspaces on $\Gr{n}{\infty}$, how can their distance/similarity be measured?

\subsection{Distance on infinite Grassmannian}
\label{sec:dist_gr_inf}
For  $\mathcal{S}, \mathcal{S}' \in \Gr{n}{d}$, the chordal distance~\cite{ ye2016schubert} is defined as 
\begin{equation}
    \label{eqn:proj_distance_main}
    d^2_{\text{chord}}\big(\mathcal{S}, \mathcal{S}'\big) = \|\mathcal{S}\mathcal{S}^\top - \mathcal{S}'\mathcal{S}'^\top\|_\text{F}^2
    = 2n - 2\|\mathcal{S}^\top\mathcal{S}'\|_\text{F}^2
\end{equation}
As $d \rightarrow \infty$ for the subspace spanned by the extended observability matrix, it is challenging to compute $d^2_{\text{chord}}$ as  $\|\mathcal{S}^\top\mathcal{S}'\|_\text{F}^2$ cannot be computed explicitly.
Let $\R^{m \times n} \ni \mO = \big[\vo_1,\vo_2,\cdots,\vo_n\big]$ be a full rank matrix. The n-dimensional subspace spanned by the columns of $\mO$ can be written as 
\begin{equation}
    \mathcal{S} = \mO\big(\mO^\top \mO)^{-1/2},
\end{equation}
as shown in \textsection\ref{app:distance_gr_inf}.
Hence, $d^2_{\text{chord}}\big(\mathcal{S}, \mathcal{S}'\big)$ can be written as:
\begin{align*}
    d^2_{\text{chord}}\big(\mathcal{S}, \mathcal{S}'\big) = 2n - 2\Big\| (\mO^\top \mO)^{-1/2} \mO^\top \mO'(\mO'^\top \mO')^{-1/2}\Big\|_{\text{F}}^2 
\end{align*}

As such, one needs to be able to compute terms such as $\mG_{1} = \big(\mO^\top \mO)$, $\mG_{2} = \big(\mO'^\top \mO')$, and $\mG_{3} =\big(\mO^\top \mO')$ for observability matrices. The lemma below shows how this is done. 

\begin{lemma} 
\label{lem:gram_main}
Let $\mA, \mA' \in \mathbb{R}^{n \times n}$ and $\mC, \mC' \in \mathbb{R}^{1 \times n}$ be the parameters of two SSMs with the extended observability matrices
%
\begin{align*}
    \mO_{\infty}(\mA, \mC) = \begin{bmatrix} \mC \\ \mC \mA \\ \mC\mA^2 \\ \vdots \end{bmatrix}, \quad
    \mO_{\infty}(\mA', \mC') = \begin{bmatrix} \mC' \\ \mC'\mA' \\ \mC'(\mA')^2 \\ \vdots \end{bmatrix}.
\end{align*}
Define the Gram matrix $\mG \in \R^{n \times n}$ as:
\begin{align*}
    \mG &= \mO_{\infty}(\mA, \mC)^\top \mO_{\infty}(\mA', \mC')  
    =\sum_{t=0}^{\infty} (\mA^\top)^t \mC^\top \mC' (\mA')^t \;.
\end{align*}
Then $\mG$ can be obtained by solving the following equation
\begin{align}
    \mA^\top \mG \mA' - \mG = - \mC^\top \mC'\;.
    \label{eqn:sylv_O_inf}
\end{align}
\end{lemma}

The form $\mA^\top \mG \mA' = \mG - \mC^\top \mC'$ is an instance of the Sylvester problem and can be solved, for example, with the Bartels-Stewart algorithm~\cite{bartels1972solution}. However, the computational complexity of solving the Sylvester algorithm is $\mathcal{O}(n^3)$, which is computationally expensive as $n$ grows larger. As pointed out in \textsection\ref{sec:ssm}, for structured SSMs (\eg, when $\mA$ is diagonal), we take advantage of the structure to reduce the computational complexity of $\mG$.
\begin{lemma}
Define $\mA_{\text{diag}}, \mA'_{\text{diag}} \in \mathbb{R}^{n \times 1}$, as diagonal elements of $\mA, \mA'$. The  Sylvester equation in \Cref{eqn:sylv_O_inf} can be written as:
\begin{align*}
    \mG \odot (\mathbf{1}_n - \mA_{\text{diag}}   \mA_{\text{diag}}'^\top)  = \mC^\top\mC'
\end{align*}
Hence, $\mG$ can be computed as:
\begin{equation}
    \mG  = \mC^\top \mC' \odot \frac{1}{\mathbf{1}_n - \mA_{\text{diag}}   \mA_{\text{diag}}'^\top} = \mO^\top\mO'
    \label{eqn:sylv_O_inf_diag}
\end{equation}
\end{lemma}
This \textbf{reduces the computational complexity} of solving the Sylvester problem from $\mathcal{O}(n^3)$ to $\mathcal{O}(n^2)$ and reduces the FLOPS count from $25n^3$ of the Bartels-Stewart algorithm~\cite{Golub1979_HessenbergSchur} to $4n^2$. In the case of $n=16$ in Vim~\cite{zhu2024VIM}, we \textbf{reduce the FLOPS count} by $\boldsymbol{100\times}$.

Putting everything together, to compute $d^2_{\text{chord}}\big(\mathcal{S}, \mathcal{S}'\big)$, we solve three Sylvester equations in the form \cref{eqn:sylv_O_inf_diag} to obtain $\mG_{1},\mG_{2}$, and $\mG_{3}$,
followed by performing matrix algebra on the resulting $\mG$ matrices based on the form of distance. 
This enables us to define a loss based on the distances of the extended observability matrices for the model of task $T-1$ and the model under training at task $T$. 

\subsection{Extension to S6}
\label{sec:extension_2_S6}
In S4D~\cite{gu2022parameterizationS4D}, and Mamba~\cite{gu2023mamba}, the dimensionality of each states are: $\mA \in \mathbb{R}^{\tau \times o \times n}$, and $\mC \in \mathbb{R}^{\tau \times n}$, where $\tau$ denotes sequence length and $o$ denotes outer dimension (or channels) of SSMs. Each sequence $\tau$ consists of $o$ LDS. Thus, the state transitions, input, and output mappings are time-variant.

To benefit from the geometry of the extended observability, we propose to generate a set of extended observability matrices, $\mathbb{O} = \{(\mO_{\infty,t}(\tilde{\mA}_t,\tilde{\mC}_t)\}_{t=1}^{\tau}$, where:
\begin{align*}
    \tilde{\mA} &= \mathrm{SN} \Big(\frac{1}{o}  \sum^{o}_{i=1} {\overline{\mA}_{i, j}}\Big) \in \mathbb{R}^{\tau \times n} \\
    \tilde{\mC} &= \mathrm{SN} \big(\mC\big) \in \mathbb{R}^{\tau \times n}
\end{align*}
and Soft-Normalization $\mathrm{SN}(x) = 2/(1+\exp{(-x)}) - 1$ is applied to enforce Schur Stability~\cite{huang2017efficient}. The states are averaged along the outer dimension, as computing distance for $\tau \times o$ pairs of SSMs is computationally prohibitive. For example, in Vim-small, this will involve $197\times384 \approx 7.6\times10^4$ pairs of SSMs, exceeding the VRAM of a single H100 GPU. Additionally, averaging along $o$ also preserves maximal variance (more information in \textsection\ref{app:G_s4}). Subsequently, we use the set $\mathbb{O}$ to regularize the model in CL.



\subsection{Inf-SSM}
\label{sec:inf_ssm}

To utilize \cref{eqn:proj_distance_main} in continual learning, we propose a combination of state distillation and regularization as illustrated in \cref{fig:inf-grass}. For the set of extended observability matrices for tasks $T-1$ and $T$ as $\mathbb{O}_{T-1}$ and $\mathbb{O}_T$ respectively, we define the loss of Inf-SSM as
\begin{align}
    \label{Eq:ISM}
    \Loss{ISM} = 
    \mathbb{E}_{\mathcal{D}_T} \bigg\{
        d^2_{\text{chord}}\Big(
            \mathbb{O}_{T-1},
            \mathbb{O}_{T}
        \Big)
    \bigg\}\;.
\end{align}
Hence, by combining with the classification loss, the total CL loss is:
\begin{equation}
    \Loss{tot} = \Loss{cls} + \lambda \Loss{ISM}
    \label{Eq:final_loss}
\end{equation}
where $\lambda$ is the regularization strength. $\Loss{cls}$ is the classification loss for the base model. The full Inf-SSM algorithm is included in \cref{Alg:Full_FGD} under \textsection\ref{app:inf_ssm_alg}. 
The minimalist design of Inf-SSM (only one additional hyperparameter $\lambda$), in contrast to some recent SOTA CL methods that may require more than 5 additional hyperparameters, makes it substantially easier to tune and deploy.




\section{Related Work}
\label{sec:back_n_theo}

\textbf{Continual Learning.}
Continual learning aims to balance \emph{stability}\footnote{\textbf{Stability: }the ability to retain past knowledge.}
and \emph{plasticity}\footnote{\textbf{Plasticity: }the ability to acquire new knowledge.}~\cite{chen2018lifelong,parisi2019continual}.
We focus on the CIL setting~\cite{hsu2018re,van2019three}, where task identities are typically unknown at test time as described in \textsection\ref{sec:problem_statement}.

A central axis in CIL is whether examples from previous tasks can be stored.
\textbf{(i)} In EFCIL, retaining past samples is prohibited due to privacy, security, or storage constraints~\cite{shin2017continual,meng2024diffclass}.
In this regime, many methods rely on regularization: EWC~\cite{kirkpatrick2017EWC},
SI~\cite{zenke2017SI}, and MAS~\cite{aljundi2018MAS} estimate parameter importance and penalize changes to critical weights.
Another line of work distills knowledge from previous models, using logits~\cite{li2017learning_LWF}, embeddings~\cite{dhar2019learning},
or features~\cite{iscen2020memory,simon2021learning}.
\textbf{(ii)} When storing or reconstructing examples is allowed, \emph{replay-based} methods interleave past and new data to mitigate forgetting, with ER~\cite{riemer2018learning} being a prominent example. Methods that combine EFCIL-style techniques with replay are often referred to as hybrid approaches, including LUCIR~\cite{hou2019learning} and X-DER~\cite{boschini2022class}. More recent hybrids, such as L2P~\cite{wang2022learning}, leverage prompt learning, while CLFD~\cite{liu2024continual} benefits from frequency domain features to further enhance CL performance. These methods are particularly effective when memory and privacy constraints permit replay. We refer to \textsection\ref{app:related_work_add} for a more extensive discussion on recent CL methods.\vspace{3pt}

\noindent
\textbf{Vision SSM.} SSMs have recently attracted considerable attention due to their theoretical ability to model infinitely long sequences, in contrast to Transformers, which are limited by fixed token lengths~\cite{gu2023mamba,zhu2024VIM,phung2024dimsum,liang2024self,liu2024robomamba}. In computer vision, Vim~\cite{zhu2024VIM} extends SSM by applying a patch embedding to images and treating each patch token as a sequential element processed through bidirectional Mamba blocks. Subsequent works like VMamba~\cite{liu2025vmamba} incorporate a 2D-selective scan to capture spatial dependencies better, while MambaVision~\cite{hatamizadeh2024mambavision} integrates SSM with Transformers to form a hybrid architecture. \vspace{3pt}

\noindent\textbf{\textit{Takeaways.}} While existing CL algorithms are generally compatible with SSMs, they treat SSMs as generic neural networks and ignore the rich underlying geometry of SSMs. Similarly, recent Mamba-based CL methods~\cite{cheng2024mamba,ZhaoMAMBACL,LiMAMBAFscil}, although effective in their specific scenarios, generally rely on storing feature embeddings from past tasks and do not explicitly exploit the geometric properties of SSMs (see \textsection\ref{app:mamba_cl}). Motivated by this gap, we study CL for SSM and, to the best of our knowledge, introduce the \textit{first} CL method that explicitly exploits the geometry of SSMs. Our approach is complementary to existing CL algorithms and can be used in both exemplar-free and replay-based scenarios.

\begin{table*}[t]
\centering
\caption{AA$(\%\uparrow)$, AIA$(\%\uparrow)$, and FM$(\%\downarrow)$ of Vim-small are reported for ImageNet-R, CIFAR-100, and Caltech-256 datasets over 5 Tasks and 10 tasks benchmarks for existing CL methods and integration with Inf-SSM. \textbf{Note:} Due to differences in batch size and epochs, the performance of different baselines cannot be directly compared.}
\footnotesize
\label{tab:cl_results_buffer}

\makebox[0.9\textwidth][c]{%
\setlength{\tabcolsep}{3pt}
\begin{tabular}{l ccc ccc ccc}
\toprule
\multirow{2}{*}{Method} &
\multicolumn{3}{c}{ImageNet-R} & \multicolumn{3}{c}{CIFAR-100} & \multicolumn{3}{c}{Caltech-256} \\
\cmidrule(lr){2-4} \cmidrule(lr){5-7} \cmidrule(lr){8-10}
& AA\std{std} & AIA\std{std} & FM\std{std} 
& AA\std{std} & AIA\std{std} & FM\std{std} 
& AA\std{std} & AIA\std{std} & FM\std{std} \\
\toprule
\multicolumn{10}{c}{5-Tasks Scenario} \\
\toprule
ER~\cite{riemer2018learning} & 30.91\std{1.53} & 56.84\std{1.21} & 64.04\std{1.01}
    & 33.00\std{1.12} & 59.66\std{0.79} & 77.05\std{1.37}
    & 49.94\std{1.54} & 70.95\std{0.98} & 55.61\std{1.60} \\
    \rowcolor{gray!15}
\ +Inf-SSM & \sota{31.26}\std{2.65} & \sota{59.69}\std{2.02} & \sota{59.45}\std{3.12}
    & \sota{33.78}\std{0.69} & \sota{60.32}\std{1.12} & \sota{76.43}\std{0.84}
    & \sota{51.39}\std{0.67} & \sota{71.91}\std{0.26} & \sota{53.23}\std{0.75} \\
\midrule
LUCIR~\cite{hou2019learning} & 31.18\std{1.49} & \sota{60.97}\std{0.94}  & 60.24\std{1.93}
   & 32.77\std{0.26} & 62.14\std{0.54} & 76.03\std{0.42}
   & 64.79\std{0.22} & 80.72\std{0.50} & 26.79\std{0.64} \\
   \rowcolor{gray!15}
\ +Inf-SSM & \sota{35.35}\std{0.89} & 59.53\std{0.85}  & \sota{52.05}\std{0.51}
   & \sota{33.10}\std{1.06} & \sota{62.58}\std{1.80} & \sota{74.67}\std{1.65}
   & \sota{68.63}\std{0.37} & \sota{80.91}\std{0.66} & \sota{19.13}\std{0.53} \\
\midrule
X-DER~\cite{boschini2022class} & 47.42\std{1.45} & 66.01\std{0.99} & 42.99\std{1.62}
    & 42.33\std{0.54} & 67.41\std{0.84} & 65.07\std{0.49}
    & 58.51\std{0.06} & 75.71\std{0.49} & 44.84\std{0.35} \\
    \rowcolor{gray!15}
\ +Inf-SSM & \sota{52.61}\std{3.01} & \sota{69.40}\std{1.58} & \sota{31.00}\std{1.90}
        & \sota{48.33}\std{0.70} & \sota{70.14}\std{0.39} & \sota{56.66}\std{0.85}
        & \sota{68.04}\std{1.48} & \sota{80.99}\std{0.71} & \sota{32.11}\std{1.67} \\
\midrule
L2P-R~\cite{wang2022learning} & 31.57\std{0.34} & 56.65\std{0.54} & 60.45\std{0.75}
        & 35.86\std{0.96} & 62.96\std{0.18} & 73.03\std{1.52}
        & 31.14\std{1.49} & 58.24\std{0.51} & 78.66\std{1.76} \\
        \rowcolor{gray!15}
\ +Inf-SSM & \sota{33.60}\std{0.54} & \sota{57.15}\std{0.81} & \sota{57.63}\std{0.43}
         & \sota{36.90}\std{0.31} & \sota{64.41}\std{0.50} & \sota{71.71}\std{0.48}
         & \sota{34.09}\std{1.17} & \sota{58.87}\std{0.41}  & \sota{74.77}\std{1.57} \\
\midrule
CLFD~\cite{liu2024continual} & 35.36\std{1.07} & 51.80\std{1.13} & 30.13\std{0.41} 
        & 33.01\std{1.35} & 57.10\std{1.96} & 69.71\std{0.82}
        & 48.86\std{1.15} & 64.78\std{1.34} & 33.63\std{0.30} \\
        \rowcolor{gray!15}
\ +Inf-SSM & \sota{37.68}\std{2.71} & \sota{53.28}\std{2.15} & \sota{28.45}\std{0.48}
         & \sota{34.20}\std{2.37} & \sota{58.31}\std{2.26} & \sota{67.96}\std{1.77}
         & \sota{50.46}\std{2.92} & \sota{65.91}\std{1.57} & \sota{32.19}\std{1.78} \\
\toprule
\multicolumn{10}{c}{10-Tasks Scenario} \\
\toprule
ER~\cite{riemer2018learning} & 21.21\std{1.01} & 50.40\std{0.78} & 70.24\std{0.97}
    & 28.03\std{0.66} & 57.05\std{1.25}  & 76.44\std{0.84}
    & 42.78\std{1.30} & 66.88\std{0.06} & 57.95\std{1.37} \\
    \rowcolor{gray!15}
\ +Inf-SSM & \sota{22.10}\std{1.94} & \sota{52.47}\std{1.12} & \sota{64.20}\std{2.47}
    &  \sota{28.75}\std{0.93} & \sota{57.22}\std{1.57} & \sota{75.78}\std{1.17}
    & \sota{43.71}\std{1.08} & \sota{67.74}\std{0.22} & \sota{56.71}\std{1.17} \\
\midrule
LUCIR~\cite{hou2019learning} & 23.63\std{1.67} & 56.84\std{1.28} & 50.36\std{1.02}
   & 24.04\std{0.96} & 53.27\std{0.73} & 80.98\std{1.20}
   & 52.24\std{1.78} & 72.88\std{0.75} & 31.45\std{1.09} \\
   \rowcolor{gray!15}
\ +Inf-SSM & \sota{28.60}\std{0.44} & \sota{57.37}\std{0.63} & \sota{42.96}\std{1.12}
   & \sota{25.72}\std{0.46} & \sota{56.33}\std{0.53} & \sota{78.75}\std{0.69}
   & \sota{55.44}\std{1.75} & \sota{74.96}\std{0.51} & \sota{28.49}\std{1.98} \\
\midrule
X-DER~\cite{boschini2022class} & 39.98\std{0.32} & 61.41\std{0.82} & 47.56\std{0.36}
    & 34.01\std{0.38} & 62.62\std{0.54}  & 69.53\std{0.29}
    & 48.30\std{0.77} & 70.77\std{0.29} & 51.57\std{0.82} \\
    \rowcolor{gray!15}
\ +Inf-SSM & \sota{44.09}\std{0.64}  & \sota{63.57}\std{0.74} & \sota{38.72}\std{0.74}
        & \sota{47.67}\std{0.80} & \sota{68.08}\std{0.22} & \sota{53.01}\std{0.83}
        & \sota{60.76}\std{1.48} & \sota{77.32}\std{0.09} & \sota{37.42}\std{1.39} \\
\midrule
L2P-R~\cite{wang2022learning} & 23.08\std{1.04} & 50.69\std{0.36} & 60.45\std{1.50}
        & 22.29\std{0.41} & 52.49\std{0.30} & 82.70\std{0.52}
        &24.13\std{1.09}  & 51.99\std{1.37} & 78.89\std{1.19} \\
        \rowcolor{gray!15}
\ +Inf-SSM & \sota{24.91}\std{1.51} & \sota{51.75}\std{0.54} & \sota{54.25}\std{1.63}
         & \sota{23.64}\std{1.70} & \sota{52.68}\std{0.39} & \sota{81.22}\std{1.94}
         & \sota{25.65}\std{1.68} & \sota{52.88}\std{0.34} & \sota{77.11}\std{1.81} \\
\midrule
CLFD~\cite{liu2024continual} & 28.90\std{0.95} & 46.72\std{0.90} & 31.62\std{1.83}
        & 29.18\std{1.36} & 54.67\std{1.75} & 62.63\std{1.74}
        & 42.12\std{0.99} & 61.57\std{0.50} & 31.09\std{0.25} \\
        \rowcolor{gray!15}
\ +Inf-SSM & \sota{30.25}\std{1.29} & \sota{48.78}\std{1.43} & \sota{30.90}\std{0.97}
         & \sota{29.72}\std{1.27}  & \sota{55.50}\std{2.38} & \sota{61.08}\std{0.56}
         & \sota{43.26}\std{0.78} & \sota{62.52}\std{0.86} & \sota{29.73}\std{0.97} \\
\bottomrule
\end{tabular}%
}
\end{table*}

\begin{table*}[t]
\centering
\footnotesize
\caption{AA$(\%\uparrow)$, AIA$(\%\uparrow)$, and FM$(\%\downarrow)$ of EFCIL methods on ImageNet-R, CIFAR-100, and Caltech-256 over 5-Tasks and 10-Tasks scenario on Vim-small. \textbf{Note:} Regularization focus is on parameter sets $(\mA , \mB, \mC)$ among all methods \textbf{except Inf-SSM}. \underline{Second-best} results are underlined.}
\label{tab:cl_results_abc}
\makebox[0.9\textwidth][c]{%
\setlength{\tabcolsep}{3pt}
\begin{tabular}{l ccc ccc ccc}
\toprule
\multirow{2}{*}{Method} &
\multicolumn{3}{c}{ImageNet-R} & \multicolumn{3}{c}{CIFAR-100} & \multicolumn{3}{c}{Caltech-256} \\
\cmidrule(lr){2-4} \cmidrule(lr){5-7} \cmidrule(lr){8-10}
& AA\std{std} & AIA\std{std} & FM\std{std} 
& AA\std{std} & AIA\std{std} & FM\std{std} 
& AA\std{std} & AIA\std{std} & FM\std{std} \\
\toprule
\multicolumn{10}{c}{5-Tasks Scenario} \\
\toprule
Seq & 38.36\std{4.47} & 61.29\std{1.76} & 56.43\std{5.14}
    & 36.68\std{1.66} & 61.25\std{1.41} & 55.00\std{2.33}
    & 37.58\std{1.08} & 60.17\std{0.95} & 71.48\std{1.52} \\
EWC~\cite{kirkpatrick2017EWC} & 45.58\std{2.72} & 65.62\std{1.16} & 47.31\std{3.18}
    & 38.25\std{1.30} & 63.12\std{1.10} & 50.71\std{1.20}
    & 42.93\std{1.64} & 64.27\std{1.31} & 64.30\std{2.24} \\
SI~\cite{zenke2017SI} & \underline{45.72}\std{3.04} & 65.18\std{1.48} & 47.21\std{3.34}
   & 37.38\std{1.40} & 61.78\std{1.05} & 53.04\std{1.35}
   & \underline{47.57}\std{0.21} & 65.29\std{1.04} & \underline{57.88}\std{0.43} \\
MAS~\cite{aljundi2018MAS} & 44.70\std{2.77} & 65.59\std{0.79} & 48.23\std{2.92}
    & 37.59\std{1.44} & 61.95\std{1.18} & 53.13\std{1.81}
    & 44.87\std{1.19} & 66.44\std{1.07} & 61.00\std{1.53} \\
LwF-ABC~\cite{li2017learning_LWF} & 45.09\std{6.58} & \underline{65.69}\std{3.17} & \underline{40.77}\std{8.26}
        & \underline{44.62}\std{2.67} & \underline{66.81}\std{1.41} & \underline{38.68}\std{3.77}
        & 46.52\std{2.66} & \underline{66.58}\std{1.08} & 59.03\std{3.43} \\
\midrule
\rowcolor{gray!15}
Inf-SSM & \sota{49.34}\std{3.36} & \sota{67.51}\std{1.47} & \sota{25.14}\std{3.86}
        & \sota{45.18}\std{2.28} & \sota{67.34}\std{1.86} & \sota{36.59}\std{3.33}
        & \sota{50.75}\std{3.16} & \sota{67.04}\std{1.43} & \sota{49.93}\std{3.61} \\
\toprule
\multicolumn{10}{c}{10-Tasks Scenario} \\
\toprule
\multirow{2}{*}{Method} &
\multicolumn{3}{c}{ImageNet-R} & \multicolumn{3}{c}{CIFAR-100} & \multicolumn{3}{c}{Caltech-256} \\
\cmidrule(lr){2-4} \cmidrule(lr){5-7} \cmidrule(lr){8-10}
& AA\std{std} & AIA\std{std} & FM\std{std} 
& AA\std{std} & AIA\std{std} & FM\std{std} 
& AA\std{std} & AIA\std{std} & FM\std{std} \\
\midrule
Seq & 32.95\std{1.71} & 55.36\std{0.70} & 58.30\std{2.19} 
    & 20.58\std{1.01} & 51.37\std{0.35} & 71.49\std{1.20}
    & 24.27\std{1.29} & 51.06\std{1.21} & 79.01\std{1.36} \\
EWC~\cite{kirkpatrick2017EWC} & \underline{41.99}\std{2.28} & 61.78\std{2.24} & 49.02\std{2.10} 
    & 22.20\std{1.11} & 53.46\std{0.35} & 63.92\std{1.28}
    & 28.35\std{0.34} & 56.64\std{0.74} & 72.73\std{0.47} \\
SI~\cite{zenke2017SI} & 41.59\std{1.17} & 61.13\std{1.43} & 46.81\std{1.00} 
   & 20.29\std{0.62} & 49.28\std{1.68} & 38.33\std{1.41}
   & 27.66\std{1.00} & 54.34\std{1.27} & 74.69\std{0.85} \\
MAS~\cite{aljundi2018MAS} & 40.10\std{1.48} & 61.30\std{0.64} & 48.18\std{1.84} 
    & 20.44\std{1.60} & 49.69\std{1.55} & 37.99\std{1.01}
    & 28.15\std{0.79} & 55.25\std{0.70} & 73.50\std{0.75} \\
LwF-ABC~\cite{li2017learning_LWF} & 41.85\std{0.82} & \underline{62.63}\std{0.92} & \underline{40.10}\std{0.61} 
        & \underline{24.39}\std{3.25} & \underline{53.48}\std{1.49} & \underline{25.29}\std{2.81}
        & \underline{35.45}\std{1.16} & \underline{59.63}\std{0.74} & \underline{64.32}\std{1.41} \\
\midrule
\rowcolor{gray!15}
Inf-SSM & \sota{43.82}\std{1.55} & \sota{62.82}\std{1.29} & \sota{36.34}\std{1.54} 
        & \sota{26.53}\std{3.10} & \sota{54.24}\std{1.87} & \sota{24.00}\std{1.80}
        & \sota{39.88}\std{2.43} & \sota{62.28}\std{2.33} & \sota{55.85}\std{4.05} \\
\bottomrule
\end{tabular}%
}
\end{table*}

\section{Experiments}
\label{sec:experiments}
In this section, we first integrate Inf-SSM into CIL methods to prove its adaptability. Next, we evaluate Inf-SSM under conditions where $(\mA, \mC)$ and $(\mA, \mB, \mC)$ are regularized, comparing its performance against existing foundational EFCIL methods across three different datasets.\vspace{3pt}

\noindent \textbf{Baselines.} 
To demonstrate the versatility of Inf-SSM, we integrate it with a diverse set of CIL methods spanning different paradigms. Specifically, we consider replay-based method: ER~\cite{riemer2018learning}, hybrid methods: LUCIR~\cite{hou2019learning} and X-DER~\cite{boschini2022class}, prompt-based method: L2P-R~\cite{wang2022learning}, and a frequency-based method: CLFD~\cite{liu2024continual} as shown in \cref{tab:cl_results_buffer}. To isolate and highlight the benefits of Inf-SSM’s geometry-aware regularization, in \cref{tab:cl_results_abc} and \cref{tab:cl_results_AC}, we further compare Inf-SSM independently against flagship EFCIL methods of EWC~\cite{kirkpatrick2017EWC}, SI~\cite{zenke2017SI}, and MAS~\cite{aljundi2018MAS} from the regularization category and  LwF~\cite{li2017learning_LWF} from the distillation category. We \textit{specifically choose} these methods to comprehensively evaluate Inf-SSM while avoiding baselines that are misaligned with our goal (more details in \textsection\ref{app:baselines}).\vspace{3pt}

\noindent
\textbf{Datasets.}
We follow previous CL studies~\cite{liang2024inflora, Mahmoodi_2023_ICCV, lubana2022quadratic, mahmoodi2025flashbacks} and use ImageNet-R~\cite{hendrycks2021many-imagenetR}, CIFAR-100~\cite{krizhevsky2009learning_cifar}, and Caltech-256~\cite{griffin_holub_perona_2022_caltech256}. For all datasets, we split the classes equally into 5 and 10 sequential tasks. More information in \textsection\ref{app:datasets}.

\noindent
\textbf{Evaluation Metrics.}
Three main evaluation metrics in CL are \textbf{(1)} Average Accuracy $\mathrm{AA}$, \textbf{(2)} Average Incremental Accuracy $\mathrm{AIA}$ and \textbf{(3)} Forgetting Measure $\mathrm{FM}$~\cite{2024WangSurvey}. $\mathrm{AA}$ and $\mathrm{AIA}$ are mainly used for evaluation of overall performance while $\mathrm{FM}$ measures the stability of the model. Ideally, a high $\mathrm{AA}$ and $\mathrm{AIA}$ with a low $\mathrm{FM}$ are desirable. More details are included in \textsection\ref{app:cl_eval_metric}.\vspace{3pt}

\noindent
\textbf{Implementation details.}
The backbone architecture is Vim-Small~\cite{zhu2024VIM} in \textbf{all experiments}. Detailed hyperparameters are included in \textsection\ref{app:add_implement}.

\newpage
\subsection{Empirical Results and Analyses}
\label{sec:results}

\textbf{Integration with existing CL methods.}
\cref{tab:cl_results_buffer} shows that on average, integrating Inf-SSM improves the corresponding baselines $\mathrm{AA}$ by \textbf{8.31\%} and reduces $\mathrm{FM}$ by \textbf{9.36\%}. This underlines the generality of Inf-SSM, as it improves all metric instances, except for the marginal decrease in $\mathrm{AIA}$ for LUCIR on 5-task ImageNet-R. The benefits become more pronounced as the number of tasks increases, as reflected by the substantially larger improvements in $\mathrm{AA}$ and $\mathrm{FM}$ compared to $\mathrm{AIA}$, indicating that Inf-SSM effectively captures the evolution of SSMs over longer task sequences. Notably, Inf-SSM remains effective when coupled with X-DER, the strongest baseline, where it improves average $\mathrm{AA}$ by \textbf{19.61\%}, and reduces $\mathrm{FM}$ by \textbf{23.16\%}.\vspace{3pt}



\noindent
\textbf{Observability state parameter regularization.}
First, we regularize the state matrices $(\mA,\mC)$ in all SSM blocks following each method's strategy. For EWC~\cite{kirkpatrick2017EWC}, SI~\cite{zenke2017SI}, and MAS~\cite{aljundi2018MAS}, regularization is applied to the parameters that directly determine $(\mA,\mC)$, while LwF~\cite{li2017learning_LWF} is implemented via distillation on $(\mA,\mC)$ (denoted LwF-AC). As reported in \cref{tab:cl_results_AC} in \textsection\ref{app:obs_param}, Inf-SSM achieves an average reduction in $\mathrm{FM}$ of \textbf{40.18\%} and an average improvement in $\mathrm{AA}$ of \textbf{21.73\%} over these baselines when regularizing $(\mA,\mC)$.\vspace{3pt}

\noindent
\textbf{Full State Parameter Regularization.} 
To stress-test Inf-SSM, we extend EWC~\cite{kirkpatrick2017EWC}, SI~\cite{zenke2017SI}, and MAS~\cite{aljundi2018MAS} to regularize all parameters involved in state-matrix formation and modify LwF~\cite{li2017learning_LWF} to distill on $(\mA,\mB,\mC)$ (denoted LwF-ABC), whereas Inf-SSM regularize $(\mA,\mC)$ only, placing it in a deliberately disadvantaged setting. Even under this configuration, \cref{tab:cl_results_abc} shows that on average, Inf-SSM reduces $\mathrm{FM}$ by \textbf{14.56\%} and improves $\mathrm{AA}$ by \textbf{6.79\%}. These results highlight the strength of Inf-SSM in mitigating CF and empirically support our claim in \cref{thm:obs_equiv_main}, that regularizing the extended observability subspace is sufficient to control the evolution of the SSM's behavior.

\subsection{Additional studies}
\label{sec:add_studies}
To complement our results in \textsection\ref{sec:results}, we have conducted \textbf{five additional studies:} \textbf{(1)} As shown in ~\cref{tab:latency_analysis}, Inf-SSM is significantly faster than recent CL methods like X-DER~\cite{boschini2022class} while still on par with simple CL methods like EWC~\cite{kirkpatrick2017EWC} without incurring any task-level overhead for computing the Fisher Information Matrix (FIM), which requires a forward-backward pass on the entire task dataset. \textbf{(2)} We applied CKD analysis in \textsection\ref{app:ckd_analysis} and revealed that earlier blocks' SSM states undergo fewer structural changes compared to deeper layers. \textbf{(3)} In \textsection\ref{app:ablation}, we applied Inf-SSM on different numbers of blocks in Vim-Small. \cref{tab:cl_results_ablation} shows that increasing the number of Vim blocks regularized will improve the performance by preventing over-regularization in shallow layers. \textbf{(4)} \textsection\ref{app:dist_speed} shows how our algorithm improves computational efficiency and stability. \textbf{(5)} \cref{Eq:ISM} ignores an important aspect of SSM, which is the input mapping defined by $\mB$. Thus, we add a Frobenius norm term to Inf-SSM to form Inf-SSM+ that regularizes $\mB$. While Inf-SSM+ outperforms Inf-SSM in some scenarios, their performance is similar, as shown in \cref{tab:cl_results_inf_ssm+} in \textsection\ref{app:inf_ssm+}, which agrees with \cref{thm:obs_equiv_main}. 

\begin{table}[t]
\centering
\footnotesize
\caption{Latency comparison between Inf-SSM and simple (EWC~\cite{kirkpatrick2017EWC}) and recent (X-DER~\cite{boschini2022class}) CL methods with Vim-small on a single NVIDIA A40 GPU.}
\label{tab:latency_analysis}
\begin{tabular}{lcc}
\toprule
\textbf{Operation} & \textbf{Mean (s)} & \textbf{Std. Dev (s)} \\
\midrule
EWC Loss (per batch) & 0.0095 & 0.0124 \\
EWC FIM (per task) & 181.5038 & 28.8674 \\
X-DER Loss (per batch) & 1.2534 & 0.0766 \\
\midrule
\rowcolor{gray!15}
Inf-SSM Loss (per batch) & 0.0960 & 0.0384 \\
\bottomrule
\end{tabular}
\vspace{-5pt}
\end{table}



\newpage
\section{Discussion and Conclusion}
\label{sec:conclusions}

\noindent\textbf{Limitations.} Inf-SSM utilizes the fact that mainstream SSMs~\cite{gu2022parameterizationS4D, gu2023mamba, zhu2024VIM, liu2025vmamba} have $\mA$ as a diagonal matrix. While a general $\mA$ will prohibit the computational shortcut described in \cref{eqn:sylv_O_inf_diag}, Inf-SSM \textbf{remains valid} for general $\mA$ as the observability Gramians can still be computed using the general Sylvester equation formulation described in \cref{def:sylvester}. Thus, the diagonal assumption is not essential for Inf-SSM, but a design choice of the mainstream SSMs' architecture that enables faster computation.\\



\noindent\textbf{Conclusion.} We propose Inf-SSM to enable SSMs to learn from sequential tasks while preserving their prior knowledge without the need for previous task distributions. We constrain model updates by regularizing their extended observability subspace to retain the model's knowledge while remaining faithful to the structure of SSMs. This is made possible by overcoming computational challenges for distances on the infinite Grassmannian. Beyond CL, Inf-SSM holds promise for broader applications. Future work may utilize Inf-SSM in knowledge distillation and model compression by integrating it with Schubert Varieties~\cite{ye2016schubert} to handle varying dimensionalities and kernel methods~\cite{huang2017efficient} within a Reproducing Kernel Hilbert Space for richer representations, potentially expanding its applicability.
\newpage
{
    \small
    \bibliography{main}

@String(ICCV= {Int. Conf. Comput. Vis.})

@String(ECCV= {Eur. Conf. Comput. Vis.})

@String(ICASSP=	{ICASSP})

@String(ICLR = {Int. Conf. Learn. Represent.})

@String(ICCV  = {ICCV})

@String(ECCV  = {ECCV})

@String(ICLR  = {ICLR})

@article{huang2017efficient,
  title={Efficient optimization for linear dynamical systems with applications to clustering and sparse coding},
  author={Huang, Wenbing and Harandi, Mehrtash and Zhang, Tong and Fan, Lijie and Sun, Fuchun and Huang, Junzhou},
  journal={Advances in neural information processing systems},
  volume={30},
  year={2017}
}

@article{turaga2011statistical,
  title={Statistical computations on Grassmann and Stiefel manifolds for image and video-based recognition},
  author={Turaga, Pavan and Veeraraghavan, Ashok and Srivastava, Anuj and Chellappa, Rama},
  journal={IEEE Transactions on Pattern Analysis and Machine Intelligence},
  volume={33},
  number={11},
  pages={2273--2286},
  year={2011},
  publisher={IEEE}
}

@article{ye2016schubert,
  title={Schubert varieties and distances between subspaces of different dimensions},
  author={Ye, Ke and Lim, Lek-Heng},
  journal={SIAM Journal on Matrix Analysis and Applications},
  volume={37},
  number={3},
  pages={1176--1197},
  year={2016},
  publisher={SIAM}
}

@article{martin2000metric,
  title={A metric for ARMA processes},
  author={Martin, Richar J.},
  journal={IEEE transactions on Signal Processing},
  volume={48},
  number={4},
  pages={1164--1170},
  year={2000},
  publisher={IEEE}
}

@article{de2002subspace,
  title={Subspace angles between ARMA models},
  author={De Cock, Katrien and De Moor, Bart},
  journal={Systems \& Control Letters},
  volume={46},
  number={4},
  pages={265--270},
  year={2002},
  publisher={Elsevier}
}

@article{ravichandran2012categorizing,
  title={Categorizing dynamic textures using a bag of dynamical systems},
  author={Ravichandran, Avinash and Chaudhry, Rizwan and Vidal, Rene},
  journal={IEEE Transactions on Pattern Analysis and Machine Intelligence},
  volume={35},
  number={2},
  pages={342--353},
  year={2012},
  publisher={IEEE}
}

@article{2024WangSurvey,
  author={Wang, Liyuan and Zhang, Xingxing and Su, Hang and Zhu, Jun},
  journal={IEEE Transactions on Pattern Analysis and Machine Intelligence}, 
  title={A Comprehensive Survey of Continual Learning: Theory, Method and Application}, 
  year={2024},
  volume={46},
  number={8},
  pages={5362-5383},
  doi={10.1109/TPAMI.2024.3367329}}

@article{kirkpatrick2017EWC,
  title={Overcoming catastrophic forgetting in neural networks},
  author={Kirkpatrick, James and Pascanu, Razvan and Rabinowitz, Neil and Veness, Joel and Desjardins, Guillaume and Rusu, Andrei A and Milan, Kieran and Quan, John and Ramalho, Tiago and Grabska-Barwinska, Agnieszka and others},
  journal={Proceedings of the national academy of sciences},
  volume={114},
  number={13},
  pages={3521--3526},
  year={2017},
  publisher={National Acad Sciences}
}

@inproceedings{zenke2017SI,
  title={Continual learning through synaptic intelligence},
  author={Zenke, Friedemann and Poole, Ben and Ganguli, Surya},
  booktitle={International conference on machine learning},
  pages={3987--3995},
  year={2017},
  organization={PMLR}
}

@inproceedings{aljundi2018MAS,
  title={Memory aware synapses: Learning what (not) to forget},
  author={Aljundi, Rahaf and Babiloni, Francesca and Elhoseiny, Mohamed and Rohrbach, Marcus and Tuytelaars, Tinne},
  booktitle={Proceedings of the European conference on computer vision (ECCV)},
  pages={139--154},
  year={2018}
}

@inproceedings{liang2024self,
  title={Self-supervised Learning for Acoustic Few-Shot Classification},
  author={Liang, Jingyong and Meyer, Bernd and Lee, Isaac Ning and Do, Thanh-Toan},
  booktitle={ICASSP 2025 - 2025 IEEE International Conference on Acoustics, Speech and Signal Processing (ICASSP)},
  year={2025}
}

@inproceedings{
    zhu2024VIM,
    title={Vision Mamba: Efficient Visual Representation Learning with Bidirectional State Space Model},
    author={Lianghui Zhu and Bencheng Liao and Qian Zhang and Xinlong Wang and Wenyu Liu and Xinggang Wang},
    booktitle={Forty-first International Conference on Machine Learning},
    year={2024}
}

@inproceedings{
    gu2023mamba,
    title={Mamba: Linear-Time Sequence Modeling with Selective State Spaces},
    author={Albert Gu and Tri Dao},
    booktitle={First Conference on Language Modeling},
    year={2024}
}

@article{gu2022parameterizationS4D,
  title={On the parameterization and initialization of diagonal state space models},
  author={Gu, Albert and Goel, Karan and Gupta, Ankit and R{\'e}, Christopher},
  journal={Advances in Neural Information Processing Systems},
  volume={35},
  pages={35971--35983},
  year={2022}
}

@inproceedings{
    gu2022efficientlyS4,
    title={Efficiently Modeling Long Sequences with Structured State Spaces},
    author={Albert Gu and Karan Goel and Christopher Re},
    booktitle={International Conference on Learning Representations},
    year={2022}
}

@article{cheng2024mamba,
  title={Mamba-{C}{L}: Optimizing Selective State Space Model in Null Space for Continual Learning},
  author={Cheng, De and Lu, Yue and He, Lingfeng and Zhang, Shizhou and Yang, Xi and Wang, Nannan and Gao, Xinbo},
  journal={arXiv preprint arXiv:2411.15469},
  year={2024}
}

@article{ZhaoMAMBACL,
  publtype={informal},
  author={Chongyang Zhao and Dong Gong},
  title={Learning Mamba as a Continual Learner},
  year={2024},
  cdate={1704067200000},
  journal={CoRR},
  volume={abs/2412.00776}
}

@article{LiMAMBAFscil,
  publtype={informal},
  author={Xiaojie Li and Yibo Yang and Jianlong Wu and Bernard Ghanem and Liqiang Nie and Min Zhang},
  title={Mamba-FSCIL: Dynamic Adaptation with Selective State Space Model for Few-Shot Class-Incremental Learning},
  year={2024},
  cdate={1704067200000},
  journal={CoRR},
  volume={abs/2407.06136}
}

@INPROCEEDINGS{Afsari2013LDSAlign,
  author={Afsari, Bijan and Vidal, René},
  booktitle={52nd IEEE Conference on Decision and Control}, 
  title={The Alignment Distance on Spaces of Linear Dynamical Systems}, 
  year={2013},
  volume={},
  number={},
  pages={1162-1167},
  doi={10.1109/CDC.2013.6760039}}

@inproceedings{rebuffi2017icarl_LwFMC,
  title={icarl: Incremental classifier and representation learning},
  author={Rebuffi, Sylvestre-Alvise and Kolesnikov, Alexander and Sperl, Georg and Lampert, Christoph H},
  booktitle={Proceedings of the IEEE conference on Computer Vision and Pattern Recognition},
  pages={2001--2010},
  year={2017}
}

@article{li2017learning_LWF,
  title={Learning without forgetting},
  author={Li, Zhizhong and Hoiem, Derek},
  journal={IEEE transactions on pattern analysis and machine intelligence},
  volume={40},
  number={12},
  pages={2935--2947},
  year={2017},
  publisher={IEEE}
}

@article{boschini2022class,
  title={Class-Incremental Continual Learning into the eXtended DER-verse},
  author={Boschini, Matteo and Bonicelli, Lorenzo and Buzzega, Pietro and Porrello, Angelo and Calderara, Simone},
  journal={IEEE Transactions on Pattern Analysis and Machine Intelligence},
  year={2022},
  publisher={IEEE}
}

@article{van2019three,
  title={Three scenarios for continual learning},
  author={Van de Ven, Gido M and Tolias, Andreas S},
  journal={arXiv preprint arXiv:1904.07734},
  year={2019}
}

@inproceedings{meng2024diffclass,
  title={Diffclass: Diffusion-based class incremental learning},
  author={Meng, Zichong and Zhang, Jie and Yang, Changdi and Zhan, Zheng and Zhao, Pu and Wang, Yanzhi},
  booktitle={European Conference on Computer Vision},
  pages={142--159},
  year={2024},
  organization={Springer}
}

@article{shin2017continual,
  title={Continual learning with deep generative replay},
  author={Shin, Hanul and Lee, Jung Kwon and Kim, Jaehong and Kim, Jiwon},
  journal={Advances in neural information processing systems},
  volume={30},
  year={2017}
}

@inproceedings{kornblith2019similarity_cka,
  title={Similarity of neural network representations revisited},
  author={Kornblith, Simon and Norouzi, Mohammad and Lee, Honglak and Hinton, Geoffrey},
  booktitle={International conference on machine learning},
  pages={3519--3529},
  year={2019},
  organization={PMLR}
}

@inproceedings{gretton2005measuring,
  title={Measuring statistical dependence with Hilbert-Schmidt norms},
  author={Gretton, Arthur and Bousquet, Olivier and Smola, Alex and Sch{\"o}lkopf, Bernhard},
  booktitle={International conference on algorithmic learning theory},
  pages={63--77},
  year={2005},
  organization={Springer}
}

@inproceedings{hendrycks2021many-imagenetR,
  title={The many faces of robustness: A critical analysis of out-of-distribution generalization},
  author={Hendrycks, Dan and Basart, Steven and Mu, Norman and Kadavath, Saurav and Wang, Frank and Dorundo, Evan and Desai, Rahul and Zhu, Tyler and Parajuli, Samyak and Guo, Mike and others},
  booktitle={Proceedings of the IEEE/CVF international conference on computer vision},
  pages={8340--8349},
  year={2021}
}

@article{krizhevsky2009learning_cifar,
  title={Learning Multiple Layers of Features from Tiny Images},
  author={Krizhevsky, A},
  journal={Master's thesis, University of Toronto},
  year={2009}
}

@article{wah_branson_welinder_perona_belongie_2011_cub, 
    title={The Caltech-UCSD Birds-200-2011 Dataset},  publisher={California Institute of Technology}, 
    author={Wah, Catherine and Branson, Steve and Welinder, Peter and Perona, Pietro and Belongie, Serge}, 
    year={2011}, 
    month={Jul} 
}

@misc{griffin_holub_perona_2022_caltech256, 
    title={Caltech 256}, 
    DOI={10.22002/D1.20087}, 
    publisher={CaltechDATA}, 
    author={Griffin, Gregory and Holub, Alex and Perona, Pietro}, 
    year={2022}, 
    month={Apr} 
}

@inproceedings{liang2024inflora,
  title={InfLoRA: Interference-Free Low-Rank Adaptation for Continual Learning},
  author={Liang, Yan-Shuo and Li, Wu-Jun},
  booktitle={Proceedings of the IEEE/CVF Conference on Computer Vision and Pattern Recognition},
  pages={23638--23647},
  year={2024}
}

@inproceedings{lubana2022quadratic,
  title={How do quadratic regularizers prevent catastrophic forgetting: The role of interpolation},
  author={Lubana, Ekdeep Singh and Trivedi, Puja and Koutra, Danai and Dick, Robert},
  booktitle={Conference on Lifelong Learning Agents},
  pages={819--837},
  year={2022},
  organization={PMLR}
}

@INPROCEEDINGS{Jia2009_imagenet,
  author={Deng, Jia and Dong, Wei and Socher, Richard and Li, Li-Jia and Kai Li and Li Fei-Fei},
  booktitle={2009 IEEE Conference on Computer Vision and Pattern Recognition}, 
  title={ImageNet: A large-scale hierarchical image database}, 
  year={2009},
  volume={},
  number={},
  pages={248-255},
  doi={10.1109/CVPR.2009.5206848}
}

@InProceedings{Mahmoodi_2023_ICCV,
    author    = {Mahmoodi, Leila and Harandi, Mehrtash and Moghadam, Peyman},
    title     = {Flashback for Continual Learning},
    booktitle = {Proceedings of the IEEE/CVF International Conference on Computer Vision (ICCV) Workshops},
    month     = {October},
    year      = {2023},
    pages     = {3434-3443}
}

@article{doretto2003dynamic,
  title={Dynamic textures},
  author={Doretto, Gianfranco and Chiuso, Alessandro and Wu, Ying Nian and Soatto, Stefano},
  journal={International journal of computer vision},
  volume={51},
  pages={91--109},
  year={2003},
  publisher={Springer}
}

@article{dhillon2008constructing_Chordal,
  title={Constructing packings in Grassmannian manifolds via alternating projection},
  author={Dhillon, Inderjit S and Heath, Jr RW and Strohmer, Thomas and Tropp, Joel A},
  journal={Experimental mathematics},
  volume={17},
  number={1},
  pages={9--35},
  year={2008},
  publisher={Taylor \& Francis}
}

@incollection{mccloskey1989catastrophic,
  title={Catastrophic interference in connectionist networks: The sequential learning problem},
  author={McCloskey, Michael and Cohen, Neal J},
  booktitle={Psychology of learning and motivation},
  volume={24},
  pages={109--165},
  year={1989},
  publisher={Elsevier}
}

@article{mcclelland1995there,
  title={Why there are complementary learning systems in the hippocampus and neocortex: insights from the successes and failures of connectionist models of learning and memory.},
  author={McClelland, James L and McNaughton, Bruce L and O'Reilly, Randall C},
  journal={Psychological review},
  volume={102},
  number={3},
  pages={419},
  year={1995},
  publisher={American Psychological Association}
}

@article{liu2025vmamba,
  title={Vmamba: Visual state space model},
  author={Liu, Yue and Tian, Yunjie and Zhao, Yuzhong and Yu, Hongtian and Xie, Lingxi and Wang, Yaowei and Ye, Qixiang and Jiao, Jianbin and Liu, Yunfan},
  journal={Advances in neural information processing systems},
  volume={37},
  pages={103031--103063},
  year={2025}
}

@article{hatamizadeh2024mambavision,
  title={Mambavision: A hybrid mamba-transformer vision backbone},
  author={Hatamizadeh, Ali and Kautz, Jan},
  journal={arXiv preprint arXiv:2407.08083},
  year={2024}
}

@article{hsu2018re,
  title={Re-evaluating continual learning scenarios: A categorization and case for strong baselines},
  author={Hsu, Yen-Chang and Liu, Yen-Cheng and Ramasamy, Anita and Kira, Zsolt},
  journal={arXiv preprint arXiv:1810.12488},
  year={2018}
}

@inproceedings{dhar2019learning,
  title={Learning without memorizing},
  author={Dhar, Prithviraj and Singh, Rajat Vikram and Peng, Kuan-Chuan and Wu, Ziyan and Chellappa, Rama},
  booktitle={Proceedings of the IEEE/CVF conference on computer vision and pattern recognition},
  pages={5138--5146},
  year={2019}
}

@inproceedings{iscen2020memory,
  title={Memory-efficient incremental learning through feature adaptation},
  author={Iscen, Ahmet and Zhang, Jeffrey and Lazebnik, Svetlana and Schmid, Cordelia},
  booktitle={Computer Vision--ECCV 2020: 16th European Conference, Glasgow, UK, August 23--28, 2020, Proceedings, Part XVI 16},
  pages={699--715},
  year={2020},
  organization={Springer}
}

@article{bartels1972solution,
  author    = {Richard H. Bartels and G. W. Stewart},
  title     = {Solution of the matrix equation {AX + XB = C}},
  journal   = {Communications of the ACM},
  volume    = {15},
  number    = {9},
  pages     = {820--826},
  year      = {1972},
  publisher = {ACM},
  doi       = {10.1145/355607.362840}
}

@article{chen2018lifelong,
  title={Lifelong Machine Learning},
  author={Chen, Zhiyuan and Liu, Bing},
  journal={Synthesis Lectures on Artificial Intelligence and Machine Learning},
  volume={12},
  number={3},
  pages={1--207},
  year={2018},
  publisher={Springer Science and Business Media LLC}
}

@article{parisi2019continual,
  title={Continual lifelong learning with neural networks: A review},
  author={Parisi, German I and Kemker, Ronald and Part, Jose L and Kanan, Christopher and Wermter, Stefan},
  journal={Neural networks},
  volume={113},
  pages={54--71},
  year={2019},
  publisher={Elsevier}
}

@article{gupta2022diagonal,
  title={Diagonal state spaces are as effective as structured state spaces},
  author={Gupta, Ankit and Gu, Albert and Berant, Jonathan},
  journal={Advances in Neural Information Processing Systems},
  volume={35},
  pages={22982--22994},
  year={2022}
}

@inproceedings{hou2019learning,
  title={Learning a unified classifier incrementally via rebalancing},
  author={Hou, Saihui and Pan, Xinyu and Loy, Chen Change and Wang, Zilei and Lin, Dahua},
  booktitle={Proceedings of the IEEE/CVF conference on computer vision and pattern recognition},
  pages={831--839},
  year={2019}
}

@ARTICLE{Golub1979_HessenbergSchur,
  author={Golub, G. and Nash, S. and Van Loan, C.},
  journal={IEEE Transactions on Automatic Control}, 
  title={A Hessenberg-Schur method for the problem AX + XB= C}, 
  year={1979},
  volume={24},
  number={6},
  pages={909-913},
  doi={10.1109/TAC.1979.1102170}}

@inproceedings{Magistri2024cold,
  author={Simone Magistri and Tomaso Trinci and Albin Soutif-Cormerais and Joost van de Weijer and Andrew D. Bagdanov},
  title={Elastic Feature Consolidation For Cold Start Exemplar-Free Incremental Learning},
  year={2024},
  cdate={1704067200000},
  booktitle={ICLR}
}

@inproceedings{gomez2024exemplar,
  title={Exemplar-free continual representation learning via learnable drift compensation},
  author={Gomez-Villa, Alex and Goswami, Dipam and Wang, Kai and Bagdanov, Andrew D and Twardowski, Bartlomiej and van de Weijer, Joost},
  booktitle={European Conference on Computer Vision},
  pages={473--490},
  year={2024},
  organization={Springer}
}

@inproceedings{
dao2024transformers,
title={Transformers are {SSM}s: Generalized Models and Efficient Algorithms Through Structured State Space Duality},
author={Tri Dao and Albert Gu},
booktitle={Forty-first International Conference on Machine Learning},
year={2024},
}

@article{vaswani2017attention,
  title={Attention is all you need},
  author={Vaswani, Ashish and Shazeer, Noam and Parmar, Niki and Uszkoreit, Jakob and Jones, Llion and Gomez, Aidan N and Kaiser, {\L}ukasz and Polosukhin, Illia},
  journal={Advances in neural information processing systems},
  volume={30},
  year={2017}
}

@inproceedings{zhu2022self,
  title={Self-sustaining representation expansion for non-exemplar class-incremental learning},
  author={Zhu, Kai and Zhai, Wei and Cao, Yang and Luo, Jiebo and Zha, Zheng-Jun},
  booktitle={Proceedings of the IEEE/CVF conference on computer vision and pattern recognition},
  pages={9296--9305},
  year={2022}
}

@inproceedings{wang2022learning,
  title={Learning to prompt for continual learning},
  author={Wang, Zifeng and Zhang, Zizhao and Lee, Chen-Yu and Zhang, Han and Sun, Ruoxi and Ren, Xiaoqi and Su, Guolong and Perot, Vincent and Dy, Jennifer and Pfister, Tomas},
  booktitle={Proceedings of the IEEE/CVF conference on computer vision and pattern recognition},
  pages={139--149},
  year={2022}
}

@inproceedings{
    riemer2018learning,
    title={Learning to Learn without Forgetting By Maximizing Transfer and Minimizing Interference},
    author={Matthew Riemer and Ignacio Cases and Robert Ajemian and Miao Liu and Irina Rish and Yuhai Tu and and Gerald Tesauro},
    booktitle={International Conference on Learning Representations},
    year={2019},
}

@article{liu2024continual,
  title={Continual learning in the frequency domain},
  author={Liu, Ruiqi and Diao, Boyu and Huang, Libo and An, Zijia and An, Zhulin and Xu, Yongjun},
  journal={Advances in Neural Information Processing Systems},
  volume={37},
  pages={85389--85411},
  year={2024}
}

@inproceedings{simon2021learning,
  title={On learning the geodesic path for incremental learning},
  author={Simon, Christian and Koniusz, Piotr and Harandi, Mehrtash},
  booktitle={Proceedings of the IEEE/CVF conference on Computer Vision and Pattern Recognition},
  pages={1591--1600},
  year={2021}
}

@inproceedings{huang2016sparse,
  title={Sparse coding and dictionary learning with linear dynamical systems},
  author={Huang, Wenbing and Sun, Fuchun and Cao, Lele and Zhao, Deli and Liu, Huaping and Harandi, Mehrtash},
  booktitle={Proceedings of the IEEE Conference on Computer Vision and Pattern Recognition},
  pages={3938--3947},
  year={2016}
}

@article{phung2024dimsum,
  title={DiMSUM: Diffusion Mamba-A Scalable and Unified Spatial-Frequency Method for Image Generation},
  author={Phung, Hao and Dao, Quan and Dao, Trung and Phan, Viet Hoang and Metaxas, Dimitris and Tran, Anh},
  journal={Advances in Neural Information Processing Systems},
  volume={37},
  pages={32947--32979},
  year={2024}
}

@article{liu2024robomamba,
  title={Robomamba: Efficient vision-language-action model for robotic reasoning and manipulation},
  author={Liu, Jiaming and Liu, Mengzhen and Wang, Zhenyu and An, Pengju and Li, Xiaoqi and Zhou, Kaichen and Yang, Senqiao and Zhang, Renrui and Guo, Yandong and Zhang, Shanghang},
  journal={Advances in Neural Information Processing Systems},
  volume={37},
  pages={40085--40110},
  year={2024}
}

@inproceedings{zhang2024motion,
  title={Motion mamba: Efficient and long sequence motion generation},
  author={Zhang, Zeyu and Liu, Akide and Reid, Ian and Hartley, Richard and Zhuang, Bohan and Tang, Hao},
  booktitle={European Conference on Computer Vision},
  pages={265--282},
  year={2024},
  organization={Springer}
}

@article{ahn2019uncertainty,
  title={Uncertainty-based continual learning with adaptive regularization},
  author={Ahn, Hongjoon and Cha, Sungmin and Lee, Donggyu and Moon, Taesup},
  journal={Advances in neural information processing systems},
  volume={32},
  year={2019}
}

@article{mahmoodi2025flashbacks,
  title={Flashbacks to harmonize stability and plasticity in continual learning},
  author={Mahmoodi, Leila and Moghadam, Peyman and Hayat, Munawar and Simon, Christian and Harandi, Mehrtash},
  journal={Neural Networks},
  volume={190},
  pages={107616},
  year={2025},
  publisher={Elsevier}
}
}
\newpage
\appendix
\begingroup
    \hypersetup{linkcolor=black}
    \let\clearpage\relax
    \let\cleardoublepage\relax
    \tableofcontents
    \hypersetup{linkcolor=red} 
\endgroup
\newpage
\section{Notations}
\label{app:notation}

\begin{table}[ht]  
\centering
\caption{Summary of Mathematical Notations}
\label{tab:notation}
\renewcommand{\arraystretch}{0.5}
\begin{tabular}{@{}ll@{}}
\toprule
\textbf{Notation} & \textbf{Description} \\
\midrule
 & General mathematical operations \\
\midrule
$\mX$ & Matrix (bold capital letter) \\
$\vx$ & Column vector (bold lowercase letter) \\
$\mI_n$ & $n \times n$ identity matrix \\
$\1_n$ & $n \times n$ ones matrix \\
$\mX_{\text{diag}}$ & Vector of diagonal elements of $\mX \in \mathbb{R}^{n \times n}$, $\mX_{\text{diag}} \in \mathbb{R}^{n \times 1}$ \\
$\overline{\mX}$ & The discretized signal of $\mX$ \\
$\| \mA \|_{\text{F}}$ & Frobenius norm of $\mA \in \mathbb{R}^{m \times n}$: $\sqrt{\sum_{i=1}^m\sum_{j=1}^n\mA[i,j]^2}$ \\
$\mA \odot \mB$ & Hadamard product of two matrices $\mA,\mB \in \R^{m \times n}$: $(\mA \odot \mB)[i,j] \;=\; \mA[i,j]\,\mB[i,j]$.\\
$\det(\mA)$ & Determinant of matrix $\mA$ \\
$\mathrm{SN}(x)$ & Soft-normalization function: $\frac{2}{1+\exp{(-x)}} - 1$~\cite{huang2017efficient} \\
$\GL{n}$ & General linear group: $\{\mP \in \mathbb{R}^{n \times n} \mid \det(\mP) \neq 0\}$ \\
$\Gr{p}{n}$ & Grassmann manifold of the set of all $p$-dimensional linear subspaces of $\R^n$ \\ 
$\Tr$ & Trace operator \\
$\vx[\cdot]$ & Discrete-time vector \\
$\vx(\cdot)$ & Continuous-time vector \\
\midrule 
 & SSM notations \\
\midrule
$\mA$ & State matrix of SSM in continuous time domain\\
$\mB$ & Input matrix of SSM in continuous time domain\\
$\mC$ & Output matrix of SSM  \\
$x$ & Input to SSM, where $x(t) \in \mathbb{R}$ \\
$\vh$ & Hidden (state) vector of SSM, where $\vh(t) \in \mathbb{R}^n$ \\
$y$ & output of SSM, where $y(t) \in \mathbb{R}$ \\
$\mO_{\infty}$ & Extended Observability matrix of SSM \\
$\tilde{\mA}$ & $\overline{\mA}$ averaged across outer dimension of SSM applied with SN~\cref{Eq:Stilde}  \\
$\tilde{\mB}$ & $\overline{\mB}$ averaged across outer dimension of SSM applied with SN~\cref{Eq:Stilde}   \\
$\tilde{\mC}$ & $\mC$ applied with SN~\cref{Eq:Stilde}   \\
$\mP$ & Any invertible $n\times n$ matrix for P-equivalence of LDS \textsection\ref{app:p_equivalence} \\
\midrule 
 & Vim specific notations \\
\midrule
$b$ & Batch size of input to SSM in Mamba and Vim \\
$\tau$ & Sequence length of SSM in Mamba and Vim\\
$o$ & Outer dimension size of SSM in Mamba and Vim \\
\midrule 
 & Grassmannian notations\\
\midrule
$\theta_i$ & Principal angle for $i$-th dimension \\
$\mG$ & Gram matrix \\
$\mathcal{S}$ & Notation of subspace, in particular $\mathcal{S} \in \mathrm{Gr}(n, m)$\\
\midrule 
 & Continual Learning notations\\
\midrule
$T$ & $T$-th task's identifier \\
$\mathcal{T}_T$ & Sequential $T$-th task \\
$N$ & Total number of task \\
$\Loss{ISM}$ & The proposed Inf-SSM regularization loss function \cref{Eq:ISM} \\
$\Loss{ISM+}$ & The proposed Inf-SSM+ regularization loss function \cref{Eq:ISM+} \\
$\mathrm{CKD}$ & Centered Kernel-Disparity,  \cref{Eq:CKD} \\
$\mathrm{AIA}$ & Average Incremental Accuracy ~\textsection\ref{app:cl_eval_metric}~\cite{2024WangSurvey} \\
$\mathrm{AA}$ & Average Accuracy~\textsection\ref{app:cl_eval_metric}~\cite{2024WangSurvey} \\
$\mathrm{FM}$ & Forgetting Measure~\textsection\ref{app:cl_eval_metric}~\cite{2024WangSurvey} \\
\bottomrule
\end{tabular}
\end{table}

\newpage
\section{Preliminary}
\label{app:prelim}

\begin{definition}[Principal Angles]
For two subspaces in $\Gr{n}{d}$, let $\mX$ and $\mZ$ be their orthonormal basis matrices. The \textbf{principal angles} $\theta_1, \theta_2, \dots, \theta_n$ between the subspaces are defined as:
\begin{align}
    \cos \theta_i = \sigma_i(\mX^\top \mZ),
    \label{eqn:princpal_angles}
\end{align}
where $\sigma_i$ is the $i$-th singular value of the matrix $ \mX^\top \mZ $, respectively.
\end{definition}

The \emph{geodesic distance} on $ \Gr{n}{d}$, inherited from its usual Riemannian metric, is given by
\begin{align}
    d_{g}(\mX, \mZ) = \sqrt{\theta_{1}^2 + \theta_{2}^2 + \cdots + \theta_{n}^2}\;.
    \label{eqn:geodesic_distance}
\end{align}
Intuitively, $\theta_{i}$ measures how much the $i$-th principal direction of $ \mX $ deviates from that of $ \mZ $, and thus this distance quantifies the ``angle" between two subspaces in a higher-dimensional setting.


\subsection{Discretization of SSMs}
\label{app:discrete_SSM}
A linear dynamic system is described by a classical \textbf{State-Space Model (SSM)} over time in the form of:
\begin{align}
    \dot{\vh}(t) &= \mA \vh(t) + \mB x(t), \notag \\
    y(t) &= \mC \vh(t)\;. 
    \label{eqn:ssm_0}
\end{align}

Here, $\vh(t) \in \R^n$ represents the \textbf{hidden state}, while $x(t), y(t) \in \R$ are the input and output of the SSM, respectively. The model is parametrized by: \textbf{1.} the {state transition matrix} $\mA \in \R^{n \times n}$, \textbf{2.} {the input mapping matrix} $\mB \in \R^{n \times 1} $, and \textbf{3.} the {output mapping matrix} $\mC \in \R^{1 \times n}$. To adopt SSM in machine learning, SSM needs to be discretized by introducing a sampling interval $\Delta \in \R$, which transforms the continuous parameters $\mA$ and $\mB$ into their discrete equivalents $\overline{\mA}$ and $\overline{\mB}$ using the Zero-Order Hold method~\cite{gu2022efficientlyS4}:
\begin{align}
    \label{eqn:A_bar}
    \overline{\mA} &= \exp(\Delta \mA) \;, \\
    \label{eqn:B_bar}
    \overline{\mB} &= (\Delta \mA)^{-1} (\exp(\Delta \mA) - \mI_n) \Delta \mB\;.
\end{align}

The discrete time domain version of \cref{eqn:ssm_0} can be written as
\begin{align}
    \textbf{h}[t] &= \overline{\mA} \textbf{h}[t-1] + \overline{\mB} x[t], \\
    y[t] &= \mC \textbf{h}[t].
    \label{eqn:discretized_ssm}
\end{align}

This discretization allows the computation of $y_t$ via a convolution operation rather than an explicit recurrence, greatly simplifying the computational complexity. \textbf{Note:} We represent $\overline{\mA}, \overline{\mB}$ as $\mA, \mB$ in the main text for better readability.

\newpage
\section{Proof}
\label{app:proofs}



\subsection{P-Equivalence}
\label{app:p_equivalence}

\begin{lemma}[P-Equivalence~\cite{doretto2003dynamic}] 
\label{lem:SSM_equiv}
Consider the SSM given by 
\begin{align*}
\begin{cases}
\vh[t] = \mA  \vh[t-1] + \mB  x[t],\\
y[t] = \mC  \vh[t],
\end{cases}
\end{align*}
where $\vh[t]\in \R^n$ is the hidden (state) vector, $x[t]\in \R$ is the input, and $y[t]\in \R$ is the output. Let $\mP$ be any invertible $n\times n$ matrix. Define
\begin{align*}
\mA' \defeq \mP\mA\mP^{-1}, 
\quad
\mB' \defeq \mP\mB, 
\quad
\mC' \defeq \mC\mP^{-1},
\end{align*}
and let the new state be $\tilde{\vh}[t] = \mP\vh[t]$. Then the SSM
\begin{align*}
\begin{cases}
\tilde{\vh}[t] = \mA' \tilde{\vh}[t-1] + \mB' x[t],\\
y[t] = \mC' \tilde{\vh}[t]
\end{cases}
\end{align*}
displays the same input-output behavior as the original system. Consequently, the triples $(\mA,\mB,\mC)$ and $(\mA',\mB',\mC')$ are said to be \emph{equivalent representations} of the same SSM.
\end{lemma}

\begin{proof}
Since $\mP$ is invertible, $\vh[t] = \mP^{-1} \tilde{\vh}[t]$. We have 
\begin{align*}
\vh[t+1] = \mP^{-1}\tilde{\vh}[t+1] = \mP^{-1}\Bigl(\mA'\tilde{\vh}[t]+\mB'x[t]\Bigr)
= \mP^{-1}\mA'\tilde{\vh}[t]+\mP^{-1}\mB'x[t] \;.
\end{align*}
Since $\tilde{\vh}[t] = \mP\vh[t]$, we get
\begin{align*}
\vh[t+1] = \mP^{-1}\mA'\mP\vh[t] + \mP^{-1}\mB'x[t] = \mA\vh[t]+ \mB x[t];,
\end{align*}
where $\mA = \mP^{-1}\mA'\mP$ and $\mB = \mP^{-1}\mB'$ by definition.
For the output, from $y[t] = \mC'\tilde{\vh}[t]$ and again using $\tilde{\vh}[t] = \mP\vh[t]$, we obtain
\begin{align*}
y[t] = \mC'\tilde{\vh}[t] = \mC'\bigl(\mP\vh[t]\bigr) = \bigl(\mC'\mP\bigr)\vh[t] = \mC\vh[t].
\end{align*}
Therefore, at each time $t$, for the same input $x[t]$, the two systems $\bigl(\mA',\mB',\mC'\bigr)$ on $\tilde{\vh}[t]$ and $\bigl(\mA,\mB,\mC\bigr)$ on $\vh[t]$
generate the same output $y[t]$.  
\end{proof}

\subsection{Invariance of the subspace spanned by the observability matrix under P-equivalence}
\label{app:obs_equiv}

\begin{theorem}[Invariance of the extended Observability under P-equivalence]
\label{thm:Obs_equiv}
Let $(\mA,\mB,\mC)$ and $(\mP\mA\mP^{-1},\mP\mB,\mC\mP^{-1})$ be two equivalent representations for an SSM for $\mP \in \GL{n}$. The extended observability subspaces satisfy:
\begin{align}
    \mathcal{S}_{\infty}(\mA', \mC') = \mathcal{S}_{\infty}(\mA, \mC)\;.
\end{align}
\end{theorem}

\begin{proof}
The extended observability matrix for the transformed system is given by
\begin{align*}
    \mO_{\infty}(\mA', \mC') = \begin{bmatrix} \mC' \\ \mC' \mA' \\ \mC \mA'^2 \\ \vdots \end{bmatrix}
    = \begin{bmatrix} \mC \mP^{-1} \\ \mC \mP^{-1} \big( \mP \mA \mP^{-1}\big) \\ \mC \mP^{-1} \big(\mP \mA \mP^{-1}\big)^2 \\ \vdots \end{bmatrix}
    = \begin{bmatrix} \mC \mP^{-1} \\ \mC \mA \mP^{-1} \\ \mC \mA^2 \mP^{-1} \\ \vdots \end{bmatrix}
    = \boxed{\mO_{\infty}(\mA, \mC) \mP^{-1}}\;.
\end{align*}

Since $\mP^{-1}$ is an invertible transformation, it does not change the span of the subspace. Therefore, we conclude that
\begin{align*}
    \mathcal{S}_{\infty}(\mA, \mC) = \mathcal{S}_{\infty}(\mP \mA \mP^{-1}, \mC \mP^{-1}) = \mathcal{S}_{\infty}(\mA', \mC') \;.
\end{align*}

\end{proof}

\subsection{Distances on Infinite Grassmannian} 
\label{app:distance_gr_inf}


Let $\mathcal{S}, \mathcal{S}' \in \Gr{n}{d}$. The chordal distance, aka projection distance, between $\mathcal{S}, \mathcal{S}'$ is defined as
\begin{align}
    \label{eqn:proj_distance}
    d^2_{\text{chord}}\big(\mathcal{S}, \mathcal{S}'\big) = \|\mathcal{S}\mathcal{S}^\top - \mathcal{S}'\mathcal{S}'^\top\|_\text{F}^2 =
    2n - 2\|\mathcal{S}^\top\mathcal{S}'\|_\text{F}^2\;.
\end{align}

Since in our case, $d \to \infty$, computing $d^2_{\text{chord}}\big(\mathcal{S}, \mathcal{S}'\big)$ is not straightforward as one cannot compute
$\|\mathcal{S}^\top\mathcal{S}'\|_\text{F}^2$ using explicit forms for $\mathcal{S}, \mathcal{S}'$. Below, we show how this can be done without the need to form 
$\mathcal{S}, \mathcal{S}' \in \Gr{n}{\infty}$ explicitly. We start by stating a result from linear algebra. 

Let $\R^{d \times n} \ni \mO = \big[\vo_1,\vo_2,\cdots,\vo_n\big]$ be a full rank matrix. The n-dimensional subspace spanned by the columns of $\mO$ can be written as 
\begin{align*}
    \mathcal{S} = \mO\big(\mO^\top \mO)^{-1/2}\;.
\end{align*}
This can be readily seen by verifying $\mathcal{S}^\top\mathcal{S} =  \big(\mO^\top \mO)^{-1/2} \mO^\top\mO\big(\mO^\top \mO)^{-1/2} = \mathbf{I}_n$. As such, we can express 
$d^2_{\text{chord}}\big(\mathcal{S}, \mathcal{S}'\big)$ as
\begin{align*}
    \label{eqn:proj_distance_baseless}
    d^2_{\text{chord}}\big(\mathcal{S}, \mathcal{S}'\big) &= 2n - 2\|\mathcal{S}^\top\mathcal{S}'\|_\text{F}^2 
    = 2n - 2\Tr\Big\{\mathcal{S}^\top\mathcal{S}'\mathcal{S}'^\top\mathcal{S}\Big\} \notag \\
    &= 2n - 2\Tr\Big\{\big(\mO^\top \mO)^{-1/2}\mO^\top  \mO'\big(\mO'^\top \mO')^{-1/2} \big(\mO'^\top \mO')^{-1/2}\mO'^\top \mO\big(\mO^\top \mO)^{-1/2}\Big\}    
    \notag \\ 
    &= 2n - 2\Tr\Big\{\big(\mO^\top \mO)^{-1}\mO^\top  \mO'\big(\mO'^\top \mO')^{-1} \mO'^\top \mO\Big\}
    \;.
\end{align*}

As such, one needs to be able to compute terms such as $\mG_{1} = \big(\mO^\top \mO)$, $\mG_{2} = \big(\mO'^\top \mO')$, $\mG_{3} =\big(\mO^\top \mO')$ and $\mG_{4} = \big(\mO'^\top \mO)$ for observability matrices. The lemma below shows how this can be done.

\begin{lemma} 
\label{lem:gram}
Let $\mA, \mA' \in \mathbb{R}^{n \times n}$ and $\mC, \mC' \in \mathbb{R}^{1 \times n}$ be the parameters of two SSMs with the extended observability matrices
\begin{align*}
    \mO_{\infty}(\mA, \mC) = \begin{bmatrix} \mC \\ \mC \mA \\ \mC\mA^2 \\ \vdots \end{bmatrix}, \quad
    \mO_{\infty}(\mA', \mC') = \begin{bmatrix} \mC' \\ \mC'\mA' \\ \mC'(\mA')^2 \\ \vdots \end{bmatrix}.
\end{align*}
Define the Gram matrix $\mG \in \R^{n \times n}$ as:
\begin{align*}
    \mG &= \mO_{\infty}(\mA, \mC)^\top \mO_{\infty}(\mA', \mC') \\
    &= \begin{bmatrix} \mC, \mC \mA , \mC\mA^2 , \hdots \end{bmatrix} \begin{bmatrix} \mC' \\ \mC'\mA' \\ \mC'(\mA')^2 \\ \vdots \end{bmatrix}\\
    &=\sum_{t=0}^{\infty} (\mA^\top)^t \mC^\top \mC' (\mA')^t \;.
\end{align*}
Then $\mG$ can be obtained by solving the following Sylvester equation
\begin{align}
    \mA^\top \mG \mA' - \mG = - \mC^\top \mC'\;.
\end{align}
\end{lemma}

\begin{proof}
Using the definition of the observability matrices,
\begin{align*}
    \mG &= \sum_{t=0}^{\infty} (\mA^\top)^t \mC^\top \mC' (\mA')^t\;.
\end{align*}
Multiplying both sides by $\mA^\top$ on the left and $\mA'$ on the right:
\begin{align*}
    \mA^\top \mG \mA' &= \sum_{t=0}^{\infty} (\mA^\top)^{t+1} \mC^\top \mC' (\mA')^{t+1} \\ 
    &= \sum_{t=1}^{\infty} (\mA^\top)^{t} \mC^\top \mC' (\mA')^{t}\\
    &= \underbrace{\sum_{t=0}^{\infty} (\mA^\top)^{t} \mC^\top \mC' (\mA')^{t}}_{\mG} - (\mA^\top)^{0} \mC^\top \mC' (\mA')^{0}\\
    &= \boxed{\mG - \mC^\top \mC'}\;.
\end{align*}
\end{proof}

The form $\mA^\top \mG \mA' = \mG - \mC^\top \mC'$ is an instance of the Sylvester problem and can be solved, for example, with the Bartels–Stewart algorithm~\cite{bartels1972solution} to obtain $\mG$. The computational complexity of solving the Sylvester algorithm is $\mathcal{O}(n^3)$. This is quite affordable as in our problem, $n$ is typically very small ($n=16$ in our experiments in Vim-small). 
\begin{definition}[Sylvester Equation]
\label{def:sylvester}
A matrix problem in the form 
\begin{align*}
    \mV \mX + \mX \mU = \mQ\;,
\end{align*}
for $\mV \in \R^{n \times n}$, $\mU \in \R^{d \times d}$ and $\mQ \in \R^{n \times d}$ over $\mX \in \R^{n \times d}$ is a Sylvester equation. A Sylvester equation has a unique solution if $\mV$ and $-\mU$ do not share any eigenvalue. 
\end{definition}

In our case, 
\begin{align}
    \mV &= \mA^\top \\
    \mU &= -\mA'^{-1} \\
    \mQ &= -\mC^\top\mC'\mA'^{-1}
\end{align}
%




\newpage
\section{SSMs, Vision Mamba and Inf-SSM}
\label{app:inf-ssm}
\subsection{Selective State-Space Models.} 
Computation of S4 is slow due to $\mA$ being parameterized as an $n \times n$ matrix. DSS~\cite{gupta2022diagonal} first shows that a diagonal state space always exists for any state space with a well-behaved matrix. S4D~\cite{gu2022parameterizationS4D} further improves the computational efficiency by computing SSM only with real numbers and re-expressing the computation of the convolution kernel as a Vandermonde matrix. Mamba~\cite{gu2023mamba}  and Vision Mamba~\cite{zhu2024VIM} (Vim) inherit the diagonal $\mA$ with further modifications on the model's structure.

A key limitation of SSMs and S4 is that they are time-invariant. As argued by Gu and Dao~\cite{gu2023mamba}, such constant dynamics prevent the model from selecting or filtering relevant information based on the input context. 
To address this issue, Selective State-Space Models (S6)~\cite{gu2023mamba} introduce \textbf{input-dependent parameters}, making the model time-varying and adaptive. Specifically, instead of keeping $\Delta, \mB, \mC$ fixed, they are modeled as functions of the input. For example, 
$\mB[t] = f_B(x(t)) = \mW_B x(t)$. 

This change significantly enriches the S6; however, it also loses the computational efficiency of the SSMs and S4, as the output can no longer be computed via convolution. The Mamba block is a selective SSM architecture inspired by S6, but with additional input-dependent transformations. Recent work in Mamba-2~\cite{dao2024transformers}, further shows that SSMs and Transformers~\cite{vaswani2017attention} are indeed dual and linked via semi-separable matrices. This potentially allows algorithms developed independently in SSMs and Transformers could be mutually beneficial to each other.

\subsection{Vision Mamba} As SSM  is implemented to handle 1-D sequence data, Vision Mamba~\cite{zhu2024VIM} transformed 2-D images into patches in 2-D. The sequence of patches is then linearly projected and added with a position embedding, forming a patch sequence. Similar to ViT, Vim utilizes a class token to capture the global information, and the entire patch sequence, including the class token, is fed into the SSM structure.
\subsection{State approximation in S4D}
\label{app:G_s4}

Let $(\mA, \mC)$ and $(\mA', \mC')$ be the tuples representing two SSMs. The SSMs could be described respectively by using the extended observability subspace formed by the tuples as discussed in \cref{lem:gram}.
\begin{align*}
    \mO_{\infty}(\mA, \mC) = \begin{bmatrix} \mC \\ \mC \mA \\ \mC\mA^2 \\ \vdots \end{bmatrix}, \quad
    \mO_{\infty}(\mA', \mC') = \begin{bmatrix} \mC' \\ \mC'\mA' \\ \mC'(\mA')^2 \\ \vdots \end{bmatrix}.
\end{align*}
The discrete-time form of the SSM equation can be expressed as:
\begin{equation}
\begin{aligned}
    \textbf{h}[t] &= \overline{\mA} \textbf{h}[t-1] + \overline{\mB} x[t], \\
    y[t] &= \mC \textbf{h}[t].
\end{aligned}
\end{equation}

In Mamba, the dimensionality of each state is: $\overline{\mA} \in \mathbb{R}^{\tau \times o \times n}$, $\overline{\mB} \in \mathbb{R}^{\tau \times o \times n}$, and $\mC \in \mathbb{R}^{\tau \times n}$. Since it is computationally infeasible to compute for the entire $\tau \times o$ pairs of states with dimensionality of $n$, we treat each $\tau$ as an independent trajectory applied over an infinite horizon. To justify our choice, as shown in Table~\ref{tab:avg_variance_A}, we measured the variance preserved when averaging over different axes. Averaging over $o$ retains more informative variance than alternatives like averaging over $\tau$ or $n$, suggesting that it captures meaningful dynamics while enabling tractable computation.

\begin{table}[h!]
\centering
\caption{Mean and standard deviation of variance preserved in  $\bar{\mathbf{A}}$ after averaging across different dimensions on ImageNet-R.}
\label{tab:avg_variance_A}
\begin{tabular}{lcc}
\toprule
\textbf{Dimension} & \textbf{Mean} & \textbf{Standard Deviation} \\
\midrule
$\tau$ & 0.0117 & 0.0085 \\
$\mathbf{o}$ & \sota{0.0315} & 0.0307 \\
$n$ & 0.0004 & 0.0003 \\
\bottomrule
\end{tabular}
\end{table}

We acknowledge that this approximation has limitations. Averaging over $o$ may fail to capture fine-grained variations. However, our empirical results in \textsection\ref{sec:experiments} and \textsection\ref{app:Experiment} suggest that this approximation is efficient without sacrificing performance.

Thus, we define:
\begin{equation}
\begin{aligned}
    \tilde{\mA} &= \mathrm{SN}(\frac{1}{o}  \sum^{o}_{i=1} {\overline{\mA}_{i, j}}) \in \mathbb{R}^{\tau \times n}, \\
    \tilde{\mB} &= \mathrm{SN}(\frac{1}{o}  \sum^{o}_{i=1} {\overline{\mB}_{i, j}}) \in \mathbb{R}^{\tau \times n}, \\
    \tilde{\mC} &= \mathrm{SN}(\mC) \in \mathbb{R}^{\tau \times n}.
\end{aligned}
\label{Eq:Stilde}
\end{equation}
 By following Huang \etal~\cite{huang2017efficient}, Soft-Normalization $\mathrm{SN}(x) = 2/(1+\exp{(-x)}) - 1$ is applied to ensure Schur Stability. Note that $\mB, \mC$ do not necessarily need to have $\mathrm{SN}$ applied to ensure Schur Stability but $\mathrm{SN}$ is applied for the consistency in magnitude across $\tilde{\mA}$ , $\tilde{\mB}$ and $\tilde{\mC}$\\

\newpage

\subsection{Derivation of Gram matrix for Vim}
For two S4D blocks represented by $\tilde{\mA}, \tilde{\mC}$ and $\tilde{\mA}', \tilde{\mC}'$. The Sylvester equation for $\tilde{\mA}$ and $\tilde{\mC}$ and $\tilde{\mA}', \tilde{\mC}'$ is given by:
\begin{align*}
    \tilde{\mA} \mG  \tilde{\mA}'^\top - \mG_{ij}  = -\tilde{\mC}^\top\tilde{\mC}'.
\end{align*}
Thus, by multiplying by -1 on both sides:
\begin{align*}
    \mG - \tilde{\mA} \mG  \tilde{\mA}'^\top   = \tilde{\mC}^\top\tilde{\mC}'.
\end{align*}
Since $\mA$ is diagonal and $\mC \in \mathbb{R}^{n \times 1}$ as mentioned in \textsection\ref{app:inf-ssm}, this allows simplification by using the Hadamard product. Thus, after simplification, the solution to the Gram matrix is:
\begin{align*}
    \mG - \tilde{\mA}_{\text{diag}}   \tilde{\mA}'^\top_{\text{diag}} \odot \mG = \tilde{\mC}^\top\tilde{\mC}'.
\end{align*}
By collecting the element-wise factor of $\mG$
\begin{align*}
    \mG \odot (\textbf{1}_n - \tilde{\mA}_{\text{diag}}   \tilde{\mA}'^\top_{\text{diag}})  = \tilde{\mC}^\top\tilde{\mC}'.
\end{align*}
Hence, by grouping the terms to form a solution for $\mG$ will obtain:
\begin{equation}
    \mG  = \tilde{\mC}^\top \tilde{\mC}' \odot \frac{1}{\textbf{1}_n - \tilde{\mA}_{\text{diag}}   \tilde{\mA}'^\top_{\text{diag}}} = \mO^\top\mO'.
    \label{Eq:S4D_G}
\end{equation}
For the formulation in ~\cref{Eq:S4D_G}, it could be broken down into four steps:
\begin{enumerate}
    \item Matrix multiplication $\tilde{\mC}^\top \tilde{\mC}$
    \item Matrix multiplication $\tilde{\mA}_{\text{diag}}   \tilde{\mA}'^\top_{\text{diag}}$
    \item Subtraction $\textbf{1}_n - \tilde{\mA}_{\text{diag}}   \tilde{\mA}'^\top_{\text{diag}}$
    \item Element-wise division of $\tilde{\mC}^\top \tilde{\mC}$ over $\textbf{1}_n - \tilde{\mA}_{\text{diag}}   \tilde{\mA}'^\top_{\text{diag}}$
\end{enumerate}
Note that $\odot$ and element-wise reciprocal could be combined as a single step by simply taking element-wise reciprocal of $\tilde{\mC}^\top \tilde{\mC}' $ by $\textbf{1}_n - \tilde{\mA}_{\text{diag}}   \tilde{\mA}'^\top_{\text{diag}}$. Thus, each outlined step have an FLOPS count of $n^2$, and hence, the total FLOPS count is $4n^2$ with computational complexity of $\mathcal{O}(n^2)$.

Hence, this \textbf{reduces the computational complexity} of solving the Sylvester problem from $\mathcal{O}(n^3)$ to $\mathcal{O}(n^2)$ and reduces the FLOPS count from $25n^3$ of the Bartels-Stewart algorithm~\cite{Golub1979_HessenbergSchur} to $4n^2$. In the case of $n=16$ in Vim~\cite{zhu2024VIM}, we \textbf{reduce the FLOPS count} by $\boldsymbol{100\times}$.

\pagebreak

\subsection{Simplified Distance on Grassmannian}
\label{app:simplified_dist}
Empirically, we could compute distance on the Grassmannian using \cref{eqn:proj_distance_baseless}. However, we observed that due to $\mA$ being diagonal in S4D, Mamba, and Vim, $\mG_i, i\in \{1,2,3,4\}$ normally have a dominant eigenvalue (principal angle) with small components of other principal directions. In particular, at $n=16$, the majority of eigenvalues are close to 0, leading $\mO$ to be very ill-conditioned and with zero determinant. Thus, we model $\mO$ as a noisy rank-1 matrix. \\

Let $\mO_1 = \va_1 \vb_1^\top \in \mathbb{R}^{\infty \times n}$ and $\mO_2 = \va_2 \vb_2^\top \in \mathbb{R}^{\infty \times n}$ by approximating $\mO$ as rank 1. We know that 
\begin{align*}
    \|\mO_i\|_F = \sqrt{\Tr(\mO_i^\top\mO_i)},
\end{align*}
where $\va_i \in \mathbb{R}^\infty$ is the column space and $\vb_i \in \mathbb{R}^n$
and 
\begin{align*}
    \|\mO_i^\top\mO_j\|_F = \sqrt{\Tr(\mO_j^\top\mO_i\mO_i^\top\mO_j)}.
\end{align*}
Hence, the principal angle between the two SSMs is
\begin{align*}
    \cos{\theta} &= \frac{\|\mO_1^\top\mO_2\|_F}{\|\mO_1\|_F\|\mO_2\|_F} = \sqrt{\frac{\Tr(\mO_2^\top\mO_1\mO_1^\top\mO_2)}{\Tr(\mO_1^\top\mO_1)\Tr(\mO_2^\top\mO_2)}}. \\
\end{align*}
Thus, the squared principal angle is
\begin{equation}
\cos^2{\theta} = {\frac{\Tr(\mG_3 \mG_4)}{\Tr(\mG_1)\Tr(\mG_2)}}.
    \label{Eq:cos_theta}
\end{equation}

Note that this formulation does not necessarily allow $\cos{\theta} = 1$ when $\mO_1 = \mO_2$ if they are not rank 1. Thus, to ensure the simplification is reasonable, we have set up a Monte-Carlo simulation of 10,000 iterations with $n=16$. For each iteration, we sample $\mA_{\text{diag}}, \mB, \mC \sim \mathcal{N}(0, I_n)$. We then sample noise to simulate weight update during a sequential training scenario by sampling $\epsilon_i \sim \mathcal{N}(0, \frac{i \cdot I_n}{25})$ for $i \in \{0, \hdots, 99\}$. Then, we measure the correlation between $i$ and our approximated $\cos{\theta}$. The Monte-Carlo test shows that for the average across 10,000 iterations, the Pearson correlation coefficient is -0.8962 with a standard deviation of 0.01515 and a mean p-value of 8.151e-30 with a standard deviation of 2.827e-28. This shows that the simplification is reasonable and applicable in Continual Learning regularization. For the problem $\cos{\theta} \neq 1$ when $\mO_1 = \mO_2$, we simply counteract it by defining:
\begin{align*}
    \cos{\theta} = 1 \quad \text{if} \ |\mA-\mA'| \leq \epsilon \ \cap |\mC-\mC'| \leq \epsilon.
\end{align*}
Formulation of ~\cref{Eq:cos_theta} is clearly faster than other distance measures on Grassmannian outlined by Ye and Lim~\cite{ye2016schubert} as it does not involve inverse and determinant at all. Next, as $\mO$ is rank deficient and ill-conditioned, computation of inverse and determinant is numerically unstable, while ~\cref{Eq:cos_theta} only involves Trace operation, which is always numerically stable.

\newpage
\section{Inf-SSM Algorithm}
\label{app:inf_ssm_alg}

 In this section, we will discuss the Inf-SSM algorithm in the Continual Learning setting. 
\begin{algorithm}[H]
\caption{\textbf{Inf-SSM State Regularization in EFCIL}}
\label{Alg:Full_FGD}
\textbf{Input:} Frozen old model $f_{T-1}(\cdot;\mW_{T-1})$, current-task dataset $\mathcal{D}_T$, regularization weight $\lambda$, learning rate $\alpha$ \\
\textbf{Output:} Updated model $f_T(\cdot;\mW_T^*)$
\begin{algorithmic}
\For{epoch $= 1, \dots, E$}
    \For{mini-batch $(\mathbf{X}, \mathbf{Y}) \sim \mathcal{D}_T$}

         \State /* \hlgrey{Forward pass in $f_{T-1}(\cdot; \mW_{T-1})$ and $f_{T}(\cdot; \mW_T)$} */

        \State Compute current-task logits $\mathbf{Z}_T \leftarrow f_T(\mathbf{X}; \mW_T)$ 
        
        \State /* \hlgreen{Compute classification loss for current task} */
        
        \State $\ell_{\mathrm{cls}} \leftarrow \Loss{cls}(\mathbf{Z}_T, \mathbf{Y})$

        \State /* \hlyellow{State extraction from past and current model} */
        \State Extract $M_{old}(\tilde{\mA}_{old}, \tilde{\mC}_{old})$ from $f_{T-1}(\cdot; \mW_{T-1})$
        \State Extract $M_{new}(\tilde{\mA}_{new}, \tilde{\mC}_{new})$ from $f_{T}(\cdot; \mW_T)$ 

        \State /* \hlblue{Compute Inf-SSM loss for reg.}*/ \Comment{\cref{Eq:final_loss}}
        \State Let $\Loss{tot}(\mW_T) = \ell_{\mathrm{cls}} + \lambda \Loss{ISM}(M_{old}, M_{new})$ 
        \State Update $\mW_T \leftarrow \mW_T - \alpha \nabla \Loss{tot}(\mW_T)$
    \EndFor
\EndFor
\State $\mW_T^* \leftarrow \mW_T$
\end{algorithmic}
\end{algorithm}

For Inf-SSM, we require the previous task model $f_{T-1}(\cdot; \mW_{T-1})$ and current task data $(\mathbf{x}, \mathbf{y}) \in (\mathcal{X}_{T}, \mathcal{Y}_{T})$ for regularizing the new model $f_{T}(\cdot; \mW_T)$. The current task data acts as a transport medium to extract the states of the previous task model and constrain the new model's weight update.

For every batch of training, we obtain the classification loss as normal. The classification loss is independent of the regularization step and thus can be set with any loss based on a specific training framework. During the forward pass, the Vim block intermediate states $\overline{\mA}, \overline{\mB}, \mC$ are generated and extracted from the old and new model. This forms the input to the Inf-SSM loss function, which will penalize the weight update in the Infinite Observability subspace in the new model. 

\par \textbf{Limitations}
As the intermediate states $\overline{\mA}, \overline{\mB}, \mC$ only exist during the computational scan process in SSM. This caused Inf-SSM to be "scan-breaking" as the intermediate state needed to be recomputed and saved outside of the scan for regularization purposes. This led to an increase in training time and VRAM requirements. However, if intermediate states could be extracted from the CUDA kernel efficiently, the computational speed of Inf-SSM should be similar to that of existing simple CL methods.

\newpage
\section{Additional Related Works}

\subsection{Hybrid Continual Learning Methods}
\label{app:related_work_add}
In the main paper, we discussed foundational EFCIL methods of Elastic Weight Consolidation (EWC)~\cite{kirkpatrick2017EWC}, Synaptic Intelligence (SI)~\cite{zenke2017SI}, and Memory-Aware Synapses 
(MAS)~\cite{aljundi2018MAS} from regularization approach and LwF~\cite{li2017learning_LWF, rebuffi2017icarl_LwFMC}. In this section, we focus on recent EFCIL methods that leverage complex hybrid architectures for better performance.

Self-sustaining representation expansion (SSRE)~\cite{zhu2022self} utilizes dynamic structural reorganization to maintain old features. This is achieved by a dual-branch structure, a main branch for fusion and a side branch for updates to retain past knowledge. The main branch will undergo distillation to transfer shared knowledge across tasks with the help of the prototype selection algorithm to selectively incorporate new knowledge into the main branch. Meanwhile, LDC~\cite{gomez2024exemplar} compensates for semantic drift via a learnable projector network that aligns features across tasks, enabling compatibility with both supervised and semi-supervised settings. At the end of each task, the trained projector is utilized to correct and update the stored prototype. LDC corrects the drift in the prototype to improve performance in EFCIL settings. EFC~\cite{Magistri2024cold} mitigates task-recency bias through prototype regularization and introduces a feature consolidation mechanism based on empirical feature drift. EFC combines the prototype pool replay, distillation, and regularization approach to mitigate catastrophic forgetting. By reducing feature drift, EFC improves the model's ability to learn new knowledge in EFCIL settings. 

While SSRE, LDC, and EFC improve performance in exemplar-free scenarios, they introduce significantly more additional components and complexity compared to foundational CL algorithms like EWC~\cite{kirkpatrick2017EWC}, LwF~\cite{li2017learning_LWF}, and Inf-SSM.

\subsection{Continual Learning in Mamba}
\label{app:mamba_cl}
For CL in vision tasks using SSMs, Mamba-CL~\cite{cheng2024mamba} enhances the stability of SSM outputs across past and current tasks by implementing orthogonality through null-space projection regularization. However, Mamba-CL needs to retain feature embeddings from all SSM modules to maintain consistency conditions for parameter updates. Meanwhile, MambaCL~\cite{ZhaoMAMBACL} incorporates meta-learning techniques to process an online data stream, enabling Mamba to function as a continual learner. MambaCL introduces selective regularization based on Mamba, linear transformers, and transformer connections. Although MambaCL recognizes the input-output relationship in SSMs, it does not account for the long-term evolution of SSM behavior, which is encoded in the extended observability subspace. Mamba-FSCIL~\cite{LiMAMBAFscil} leverages a dual selective SSM projector to learn shifts in feature-space distribution and employs class-sensitive selective scan to improve model stability by reducing inter-class interference. Mamba-FSCIL requires the storage of intermediate feature embeddings from the old tasks distribution to utilize these features to improve the model's stability. From an SSM geometrical perspective, Mamba-FSCIL focuses on input-dependent parameters $\mB, \mC, \Delta$ in class-sensitive selective scan, where the key input state matrix $\mA$ is not leveraged. 

In short, while these recent works achieve impressive results within their respective settings, they are not directly comparable to Inf-SSM, as our focus lies in developing a fundamental continual learning algorithm that is flexible and complementary to mainstream CL methods rather than tailored to a specific scenario.
\newpage
\section{Experiments}
\label{app:Experiment}
\subsection{Datasets}
\label{app:datasets}
We conducted the experiments on four different datasets, namely: (1) \textbf{ImageNet-R}~\cite{hendrycks2021many-imagenetR} with 30,000 images distributed unevenly in 200 classes of renditions of ImageNet~\cite{Jia2009_imagenet}. (2) \textbf{CIFAR-100}~\cite{krizhevsky2009learning_cifar}, a balanced dataset of 100 classes consisting of 50,000 training images. (3) \textbf{Caltech-256} with 30,607 images from 256 classes~\cite{griffin_holub_perona_2022_caltech256}. We considered the first 250 classes for equal task partitioning. Each dataset is partitioned equally in terms of the number of classes across 5-task and 10-task scenarios.

In addition, we also utilized (4) \textbf{CUB-200-2011}~\cite{wah_branson_welinder_perona_belongie_2011_cub} of the extended version 2011 with 11,788 images distributed unevenly across 200 classes for ablation studies purposes.

\subsection{Continual Learning Evaluation Metrics}
\label{app:cl_eval_metric}
The notation of $\mathrm{AA}$ and $\mathrm{AIA}$ for test dataset of task $\mathcal{T}_j$ after training on $\mathcal{T}_k$ tasks where $j\leq k$ are as follow~\cite{2024WangSurvey}:
\begin{equation}
        \mathrm{AA}_k = \frac{1}{k}\sum_{j=1}^{k} a_{k,j},
\end{equation}
\begin{equation}
    \mathrm{AIA}_k = \frac{1}{k}\sum_{i=1}^{k} \mathrm{AA}_i.
\end{equation} 
where $a_{k,j}$ is the accuracy of the test dataset for task $\mathcal{T}_j$ after the model is trained on task $\mathcal{T}_k$. Meanwhile $\mathrm{FM}$ at task $k$ is defined as:
\begin{equation}
    \mathrm{FM}_k = \frac{1}{k-1}\sum^{k-1}_{j=1} \underset{i \in {1, ..., k-1}}{\max}(a_{i,j}-a_{k,j}).
\end{equation}

\subsection{Observability state parameter regularization}
\label{app:obs_param}

In this section, we present our experiment on observability state parameter regularization. For this experiment, all baseline methods, EWC~\cite{kirkpatrick2017EWC}, SI~\cite{zenke2017SI}, and MAS~\cite{aljundi2018MAS} from the regularization-based category and  LwF~\cite{li2017learning_LWF} from distillation methods are applied on weights directly contributing to the formation of state matrices $\mA, \mC$. 

\begin{table*}[h!]
\centering
    \caption{AA$(\%\uparrow)$, AIA$(\%\uparrow)$, and FM$(\%\downarrow)$ of EFCIL methods on Vim-small are reported for ImageNet-R Dataset over 5 Tasks and 10 tasks benchmarks. \textbf{Note:} Regularization focus is on parameter sets $(\mA , \mC)$ among all methods. \underline{Second-best} results are underlined.}
    \label{tab:cl_results_AC}
    \renewcommand{\arraystretch}{0.75}
    \begin{tabular}{l ccc ccc}
      \toprule
      \multirow{2}{*}{Method} &
      \multicolumn{3}{c}{ImageNet-R 5 task} & \multicolumn{3}{c}{ImageNet-R 10 task} \\
      \cmidrule(lr){2-4} \cmidrule(lr){5-7}
      & AA\std{std} & AIA\std{std} & FM\std{std} & AA\std{std} & AIA\std{std} & FM\std{std}  \\
      \midrule
      Seq     & 40.57\std{1.36} & 62.12\std{0.42} & 53.94\std{1.63} & 31.96\std{1.82} & 55.57\std{0.42} & 58.95\std{1.15} \\
      EWC~\cite{kirkpatrick2017EWC}     & 42.25\std{4.02} & 63.55\std{1.17} & 51.20\std{4.79} & \underline{33.92}\std{1.82} & 56.39\std{0.23} & 56.68\std{2.34} \\
      SI~\cite{zenke2017SI}      & 41.78\std{1.64} & 62.76\std{1.10} & 52.17\std{1.63} & 33.85\std{0.64} & 55.29\std{0.27} & 56.59\std{0.77} \\
      MAS~\cite{aljundi2018MAS}     & 40.32\std{0.93} & 62.31\std{0.52} & 53.87\std{1.44} & 33.39\std{0.79} & 55.67\std{0.47} & 57.60\std{1.10} \\
      LwF-AC~\cite{li2017learning_LWF}  & \underline{43.18}\std{1.57} & \underline{65.97}\std{0.45} & \underline{47.66}\std{1.81} & 33.00\std{1.79} & \underline{58.43}\std{0.92} & \underline{54.33}\std{1.95} \\
      \cmidrule(lr){1-7} 
      \rowcolor{gray!15}
      Inf-SSM & \sota{49.34}\std{3.36} & \sota{67.51}\std{1.47} & \sota{25.14}\std{3.86} & \sota{43.82}\std{1.55} & \sota{62.82}\std{1.29} & \sota{36.34}\std{1.54} \\
      \bottomrule
    \end{tabular}
\end{table*}

As shown in~\cref{tab:cl_results_AC}, Inf-SSM achieves superior performance across both benchmarks. Compared to the best baseline, Inf-SSM reduces FM by 47.25\% and 33.11\%, while improving $\mathrm{AA}$ by 14.27\% and 29.19\% for 5- and 10-task settings, respectively. These results demonstrate that, under fair comparisons, Inf-SSM outperforms foundational EFCIL methods owing to its ability to capture the underlying geometry of Linear-Input-Varying SSMs.

\newpage
\section{Additional studies}

\subsection{Centered Kernel Disparity states analysis}
\label{app:ckd_analysis}
Since the evolution of SSM internal states over the task sequence is not well understood, we first analyze their structural changes across EFCIL tasks using similarity measures. A prominent choice is Centered Kernel Alignment (CKA)~\cite{kornblith2019similarity_cka}.
\begin{equation}
    \mathrm{CKA}(\mW_1, \mW_2) = \frac{\mathrm{HSIC}(\mW_1, \mW_2)}{\sqrt{\mathrm{HSIC}(\mW_1, W_1)\mathrm{HSIC}(\mW_2, \mW_2)}}.
\end{equation}
where $\mathrm{HSIC}$ is Hilbert-Schmidt Independence Criterion~\cite{gretton2005measuring}. To achieve positive correlation with forgetting, we redefined the CKA of CL settings as Centered Kernel Disparity (CKD), where
\begin{equation}
    \mathrm{CKD}(\mW_1, \mW_2) = 1 - \mathrm{CKA}(\mW_1, \mW_2).
    \label{Eq:CKD}
\end{equation}
Instead of measuring weights as commonly used, CKD will be utilized to analyze state evolution in CL settings across sequential tasks and across different SSM layers in VIM.

\begin{figure}[H]
    \centering
    \includegraphics[width=0.75\linewidth]{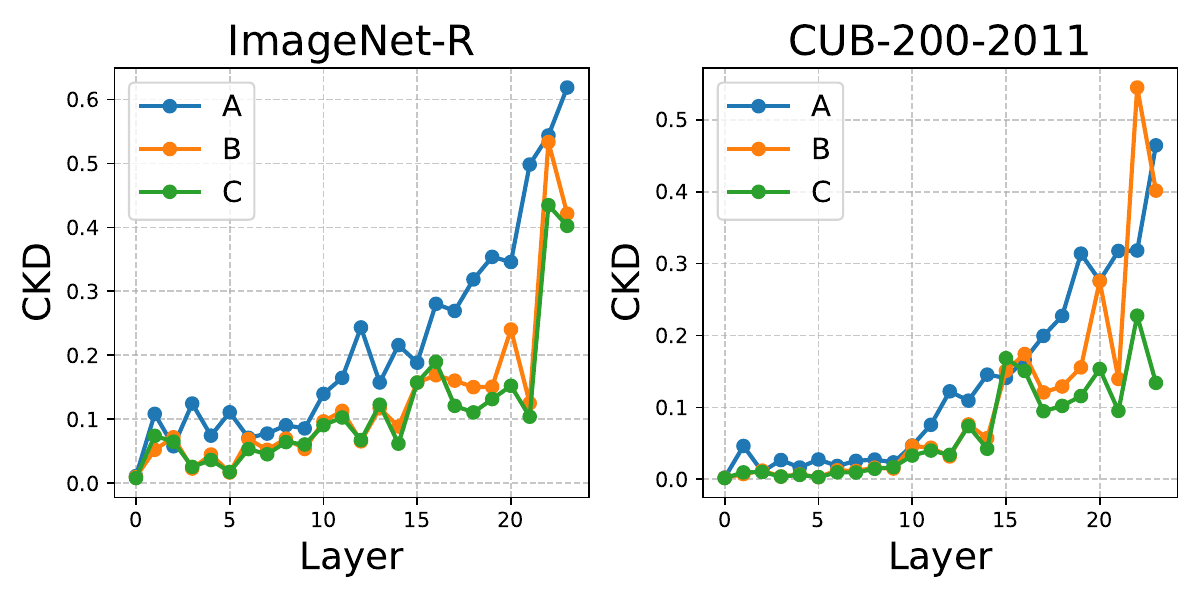}
    \caption{CKD analysis on Vim-small with ImageNet-R and CUB-200-2011 over 10 tasks, EFCIL settings for each of the 24 layers of SSM blocks.}
    \label{fig:ckd_layer}
\end{figure}
\cref{fig:ckd_layer} shows that SSM state changes most at the last few layers. However, regularization is applied across all layers as we deem that small changes in activation in unregulated early layers will be amplified across the subsequent layers and lead to large CF. 

\newpage
\subsection{Inf-SSM ablation studies}
\label{app:ablation}
In this study, we aim to investigate whether applying Inf-SSM to all Vim blocks is necessary. As discussed in \textsection\ref{app:ckd_analysis}, the SSM state matrices in the final Vim block undergo significant changes compared to the earlier blocks. Thus, in this ablation study, we apply Inf-SSM to different numbers of blocks in Vim-small. As presented in \cref{tab:cl_results_ablation}, $\mathrm{AIA}$ and $\mathrm{AA}$ improve with the number of blocks regularized. Interestingly, $\mathrm{FM}$ increases with the number of blocks regularized, as we expect the model stability to improve instead of degrading when more blocks are applied with Inf-SSM.
\begin{table}[ht]
\centering
\setlength{\tabcolsep}{4.5pt}
\caption{Ablation study on how many layers of Inf-SSM need to be applied in 10 tasks ImageNet-R and CUB-200-2011.}
\label{tab:cl_results_ablation}
\begin{tabular}{l ccc ccc}
\toprule
\multirow{2}{*}{Vim Block applied} &
\multicolumn{3}{c}{ImageNet-R} & \multicolumn{3}{c}{CUB-200-2011} \\
\cmidrule(lr){2-4}  \cmidrule(lr){5-7}
& AIA & AA & FM & AIA & AA & FM  \\
\midrule
Final Block & 60.69\std{2.80} & 42.60\std{5.53} & \sota{23.18}\std{4.70} & 28.03\std{1.36} & 11.04\std{1.36} & \sota{1.19}\std{1.39} \\
Final 12 Blocks & 62.66\std{1.78} & 43.24\std{2.74} & 36.16\std{2.98} & 40.16\std{1.38} & 17.56\std{2.13} & 24.20\std{2.45} \\
All 24 Blocks  & \sota{62.82}\std{1.29} & \sota{43.82}\std{1.55} & 36.34\std{1.54} & \sota{42.11}\std{0.83} & \sota{18.97}\std{0.90} & 45.88\std{0.50} \\
\bottomrule
\end{tabular}
\end{table}
This observation can be explained by the fact that regularizing fewer layers means that earlier layers tend to over-regularize to minimize Inf-SSM loss in the regularized layers, leading to lower overall plasticity in the model. Thus, as the initial accuracy when a new task is learned is low, the final $\mathrm{FM}$ decreases as the number of blocks regularized decreases. 

\subsection{Simplified distance performance of Inf-SSM}
In this part, Inf-SSM distance as derived in \cref{Eq:cos_theta} is compared to other distances on the Grassmannian manifold in terms of speed. We randomly sampled $\mA\in\mathbb{R}^{n\times n}$ and $\mC \in\mathbb{R}^{1\times n}$ matrices from a Gaussian distribution with $n=100$. The experiment is conducted on a single A40 40GB VRAM GPU with 10000 Monte Carlo simulation iterations. Four distances are compared: Inf-SSM, Fubini-Study~\cite{ ye2016schubert}, Martin~\cite{ravichandran2012categorizing, huang2017efficient}, Binet-Cauchy~\cite{ ye2016schubert}.
\label{app:dist_speed}
\begin{table}[ht]
    \centering
    \caption{Monte Carlo simulation of 10000 iterations on a single A40 GPU with $\mA\in\mathbb{R}^{100\times 100}$ and $\mC \in \mathbb{R}^{1\times 100}$ for comparisons of computational time between Inf-SSM and other distance metrics in \textsection\ref{app:chordal_equiv}}
    \label{tab:rspeed}
    \begin{tabular}{lll}
      \toprule
      Metric & t (s) & $\pm$Std. Dev. (s) \\
      \midrule
      Martin & 0.0439 &  $8.3e{-3}$ \\
      Fubini-Study & 0.0422 & $5.9e{-3}$ \\
      Binet-Cauchy & 0.0430 & $5.9e{-3}$  \\
      Inf-SSM  & \sota{0.0004} & $\sota{4.5e{-3}}$ \\
      \bottomrule
    \end{tabular}
\end{table}
As shown in ~\cref{tab:rspeed}, Inf-SSM decreases computational time by 99.05\% when compared to Fubini-Study distance with 23.73\% less standard deviation. The benefits of computational speed will be even more significant when taking into account backpropagation during training of Vim, as computation of the gradient over determinant or inverse operations is costly.

\newpage
\subsection{Inf-SSM+}
\label{app:inf_ssm+}
Inf-SSM, as shown in \cref{Eq:ISM} ignores one important aspect of SSM, which is the input mapping defined by $\mB$. This creates an unregularized path in the SSM algorithm and might lead to CF. Thus, we include $\mB$\footnote{Note for readability, $\mB$ refers to $\tilde{\mB}$ derived in \textsection\ref{app:G_s4}.} in the loss by adding a Frobenius norm regularization term. We define this variant of Inf-SSM as Inf-SSM+, more specifically:
\begin{align}
    \label{Eq:ISM+}
    \Loss{ISM+} = \Loss{ISM} + \gamma\mathbb{E}_{\mathcal{D}_T} \big\|\mB_{T-1} - \mB_{T}\big\|_{\text{F}}^2\;.
\end{align}
where $\gamma$ controls the regularization strength on $\mB$. 

\begin{table*}[h!]
\centering

\caption{AA$(\%\uparrow)$, AIA$(\%\uparrow)$, and FM$(\%\downarrow)$ of Inf-SSM and Inf-SSM+ on ImageNet-R, CIFAR-100, and Caltech-256 over 5-Tasks and 10-Tasks scenario on Vim-small.}
\label{tab:cl_results_inf_ssm+}
\setlength{\tabcolsep}{2pt}
\begin{tabular}{l ccc ccc ccc}
\toprule
\multirow{2}{*}{Method} &
\multicolumn{3}{c}{ImageNet-R} & \multicolumn{3}{c}{CIFAR-100} & \multicolumn{3}{c}{Caltech-256} \\
\cmidrule(lr){2-4} \cmidrule(lr){5-7} \cmidrule(lr){8-10}
& AA\std{std} & AIA\std{std} & FM\std{std} 
& AA\std{std} & AIA\std{std} & FM\std{std} 
& AA\std{std} & AIA\std{std} & FM\std{std} \\
\toprule
\multicolumn{10}{c}{5-Tasks Scenario} \\
\toprule

Inf-SSM & \sota{49.34}\std{3.36} & \sota{67.51}\std{1.47} & \sota{25.14}\std{3.86}
        & 45.18\std{2.28} & 67.34\std{1.86} & 36.59\std{3.33}
        & 50.75\std{3.16} & 67.04\std{1.43} & 49.93\std{3.61} \\
Inf-SSM+ & 47.43\std{6.57} & 66.36\std{3.12} & 27.81\std{8.86}
         & \sota{46.87}\std{2.85} & \sota{68.04}\std{2.02} & \sota{31.08}\std{3.64}
         & \sota{51.10}\std{2.68} & \sota{68.17}\std{0.98} & \sota{49.44}\std{3.27} \\
\toprule
\multicolumn{10}{c}{10-Tasks Scenario} \\
\toprule

\multirow{2}{*}{Method} &
\multicolumn{3}{c}{ImageNet-R} & \multicolumn{3}{c}{CIFAR-100} & \multicolumn{3}{c}{Caltech-256} \\
\cmidrule(lr){2-4} \cmidrule(lr){5-7} \cmidrule(lr){8-10}
& AA\std{std} & AIA\std{std} & FM\std{std} 
& AA\std{std} & AIA\std{std} & FM\std{std} 
& AA\std{std} & AIA\std{std} & FM\std{std} \\
\midrule
Inf-SSM & \sota{43.82}\std{1.55} & 62.82\std{1.29} & \sota{36.34}\std{1.54} 
        & \sota{26.53}\std{3.10} & \sota{54.24}\std{1.87} & 24.00\std{1.80}
        & \sota{39.88}\std{2.43} & 62.28\std{2.33} & \sota{55.85}\std{4.05} \\
Inf-SSM+ & 43.62\std{1.48} & \sota{63.29}\std{0.74} & 37.69\std{1.49} 
         & 24.70\std{3.33} & 53.38\std{1.82} & \sota{23.85}\std{2.26}
         & 39.24\std{1.85} & \sota{62.33}\std{1.96} & 56.50\std{3.07} \\
\bottomrule
\end{tabular}
\end{table*}
Averaged over both task splits and all datasets, Inf-SSM+ only achieves a negligible AIA gain of $0.04\%$ and a small FM reduction of $0.19\%$ over Inf-SSM. Meanwhile, Inf-SSM+ incurs a $1.40\%$ drop in AA when compared to Inf-SSM.  Overall, explicitly regularizing $\mB$ offers limited benefit while increasing the computational cost.

\newpage
\subsection{Vim-tiny}
\label{app;vim_tiny}

To validate that Inf-SSM is adaptable to different model sizes, especially on small models, we validated Inf-SSM along with EFCIL baselines of EWC, SI, MAS, and LwF on Vim-tiny. Vim-tiny only has 7M parameters in comparison to 26M of Vim-small~\cite{zhu2024VIM}.

\begin{table*}[h!]
\centering

\caption{AA$(\%\uparrow)$, AIA$(\%\uparrow)$, and FM$(\%\downarrow)$ of EFCIL methods on ImageNet-R 5-Task and 10-Task scenario on \textbf{Vim-tiny}. \textbf{Note:} Regularization focus is on parameter sets $(\mA , \mB, \mC)$ among all methods \textbf{except Inf-SSM}. \underline{Second-best} results are underlined.}
\label{tab:cl_results_vim_tiny}
\renewcommand{\arraystretch}{0.9}
\begin{tabular}{l ccc ccc}
\toprule
\multirow{2}{*}{Method} &
\multicolumn{3}{c}{ImageNet-R 5-task} & \multicolumn{3}{c}{ImageNet-R 10-task}\\
\cmidrule(lr){2-4}  \cmidrule(lr){5-7}
& AA\std{std} & AIA\std{std} & FM\std{std} & AA\std{std} & AIA\std{std} & FM\std{std} \\
\midrule
Seq      & 25.68\std{1.47} & 50.28\std{1.35} & 63.05\std{2.45}
        & 17.58\std{1.34} & 40.07\std{1.21} & 63.00\std{1.35} \\
EWC~\cite{kirkpatrick2017EWC}        & 35.73\std{1.94} & 57.62\std{1.27} & 48.34\std{1.56} 
       & 23.14\std{3.17} & 47.38\std{1.49} & 55.45\std{3.24} \\
SI~\cite{zenke2017SI}        & 25.55\std{0.60} & 50.05\std{1.01} & 63.38\std{0.35}  
       & 17.15\std{1.22} & 39.41\std{1.26} & 64.65\std{1.27} \\
MAS~\cite{aljundi2018MAS}       & 27.94\std{0.85} & 52.47\std{1.24} & 60.15\std{0.38}  
       & 22.41\std{1.48} & 43.43\std{1.19} & 60.23\std{1.76} \\
LwF-ABC~\cite{li2017learning_LWF}       & \underline{37.17}\std{1.99} & \sota{58.87}\std{1.85} & \underline{35.10}\std{2.30} 
       & \underline{26.17}\std{2.33} & \sota{49.25}\std{1.23} & \underline{32.25}\std{5.68} \\
\midrule
\rowcolor{gray!15}
Inf-SSM   & \sota{39.85}\std{1.15}  & \underline{58.74}\std{1.06}  &
\sota{23.33}\std{1.24}       & \sota{27.92}\std{1.34} & \sota{49.25}\std{1.07} & \sota{25.18}\std{1.55} \\
\bottomrule
\end{tabular}
\end{table*}
On average, Inf-SSM outperforms the previous best by $6.95\%$ for $\mathrm{AA}$ and reduces $\mathrm{FM}$ by 27.74\%. These results are consistent with our experiments on Vim-small, where Inf-SSM is particularly effective at retaining past knowledge, and its advantage grows as the number of learned tasks increases. Although LwF-ABC attains a slightly higher $\mathrm{AIA}$ than Inf-SSM, as reported in \cref{tab:cl_results_vim_tiny}, the gap is marginal and well within the standard deviation.

\newpage

\subsection{Additional EFCIL Baseline}
For an even more comprehensive empirical validation, we have adapted and re-implemented the more recent uncertainty-based method UCL~\cite{ahn2019uncertainty} in Vim-small. The comparison below shows that Inf-SSM consistently outperforms UCL across both 5-task and 10-task settings for ImageNet-R, CIFAR-100, and Caltech-256.

\begin{table*}[h]
\centering
\footnotesize
\caption{AA$(\%\uparrow)$, AIA$(\%\uparrow)$, and FM$(\%\downarrow)$ of EFCIL methods on ImageNet-R, CIFAR-100, and Caltech-256 over 5-Tasks and 10-Tasks scenario on Vim-small. \textbf{Note:} Regularization focus is on parameter sets $(\mA , \mB, \mC)$ among all methods \textbf{except Inf-SSM}. \underline{Second-best} results are underlined.}
\label{tab:cl_results_abc_ucl}
\renewcommand{\arraystretch}{1.0}
\setlength{\tabcolsep}{2.5pt}
\begin{tabular}{l ccc ccc ccc}
\toprule
\multirow{2}{*}{Method} &
\multicolumn{3}{c}{ImageNet-R} & \multicolumn{3}{c}{CIFAR-100} & \multicolumn{3}{c}{Caltech-256} \\
\cmidrule(lr){2-4} \cmidrule(lr){5-7} \cmidrule(lr){8-10}
& AA\std{std} & AIA\std{std} & FM\std{std} 
& AA\std{std} & AIA\std{std} & FM\std{std} 
& AA\std{std} & AIA\std{std} & FM\std{std} \\
\toprule
\multicolumn{10}{c}{5-Tasks Scenario} \\
\toprule
Seq & 38.36\std{4.47} & 61.29\std{1.76} & 56.43\std{5.14}
    & 36.68\std{1.66} & 61.25\std{1.41} & 55.00\std{2.33}
    & 37.58\std{1.08} & 60.17\std{0.95} & 71.48\std{1.52} \\
EWC~\cite{kirkpatrick2017EWC} & 45.58\std{2.72} & 65.62\std{1.16} & 47.31\std{3.18}
    & 38.25\std{1.30} & 63.12\std{1.10} & 50.71\std{1.20}
    & 42.93\std{1.64} & 64.27\std{1.31} & 64.30\std{2.24} \\
SI~\cite{zenke2017SI} & 45.72\std{3.04} & 65.18\std{1.48} & 47.21\std{3.34}
   & 37.38\std{1.40} & 61.78\std{1.05} & 53.04\std{1.35}
   & \underline{47.57}\std{0.21} & 65.29\std{1.04} & \underline{57.88}\std{0.43} \\
MAS~\cite{aljundi2018MAS} & 44.70\std{2.77} & 65.59\std{0.79} & 48.23\std{2.92}
    & 37.59\std{1.44} & 61.95\std{1.18} & 53.13\std{1.81}
    & 44.87\std{1.19} & 66.44\std{1.07} & 61.00\std{1.53} \\
\rowcolor{green!15}
UCL~\cite{ahn2019uncertainty} & \underline{48.03}\std{1.15} & \underline{67.19}\std{0.22} & 43.90\std{1.29}
    & 39.48\std{4.95} & 62.62\std{2.89} & 46.52\std{7.15}
    & 44.74\std{2.66} & 65.17\std{0.42} & 62.48\std{3.80} \\
LwF-ABC~\cite{li2017learning_LWF} & 45.09\std{6.58} & 65.69\std{3.17} & \underline{40.77}\std{8.26}
        & \underline{44.62}\std{2.67} & \underline{66.81}\std{1.41} & \underline{38.68}\std{3.77}
        & 46.52\std{2.66} & \underline{66.58}\std{1.08} & 59.03\std{3.43} \\
\midrule
\rowcolor{gray!15}
Inf-SSM & \sota{49.34}\std{3.36} & \sota{67.51}\std{1.47} & \sota{25.14}\std{3.86}
        & \sota{45.18}\std{2.28} & \sota{67.34}\std{1.86} & \sota{36.59}\std{3.33}
        & \sota{50.75}\std{3.16} & \sota{67.04}\std{1.43} & \sota{49.93}\std{3.61} \\
\toprule
\multicolumn{10}{c}{10-Tasks Scenario} \\
\toprule
\multirow{2}{*}{Method} &
\multicolumn{3}{c}{ImageNet-R} & \multicolumn{3}{c}{CIFAR-100} & \multicolumn{3}{c}{Caltech-256} \\
\cmidrule(lr){2-4} \cmidrule(lr){5-7} \cmidrule(lr){8-10}
& AA\std{std} & AIA\std{std} & FM\std{std} 
& AA\std{std} & AIA\std{std} & FM\std{std} 
& AA\std{std} & AIA\std{std} & FM\std{std} \\
\midrule
Seq & 32.95\std{1.71} & 55.36\std{0.70} & 58.30\std{2.19} 
    & 20.58\std{1.01} & 51.37\std{0.35} & 71.49\std{1.20}
    & 24.27\std{1.29} & 51.06\std{1.21} & 79.01\std{1.36} \\
EwC~\cite{kirkpatrick2017EWC} & \underline{41.99}\std{2.28} & 61.78\std{2.24} & 49.02\std{2.10} 
    & 22.20\std{1.11} & 53.46\std{0.35} & 63.92\std{1.28}
    & 28.35\std{0.34} & 56.64\std{0.74} & 72.73\std{0.47} \\
SI~\cite{zenke2017SI} & 41.59\std{1.17} & 61.13\std{1.43} & 46.81\std{1.00} 
   & 20.29\std{0.62} & 49.28\std{1.68} & 38.33\std{1.41}
   & 27.66\std{1.00} & 54.34\std{1.27} & 74.69\std{0.85} \\
MAS~\cite{aljundi2018MAS} & 40.10\std{1.48} & 61.30\std{0.64} & 48.18\std{1.84} 
    & 20.44\std{1.60} & 49.69\std{1.55} & 37.99\std{1.01}
    & 28.15\std{0.79} & 55.25\std{0.70} & 73.50\std{0.75} \\
\rowcolor{green!15}
UCL~\cite{ahn2019uncertainty} & 40.10\std{1.87} & 60.45\std{1.01} & 48.84\std{1.94}
    & 21.71\std{1.19} & 50.35\std{0.50} & 29.16\std{1.00}
    & 33.16\std{3.43} & 57.95\std{1.77} & 69.07\std{4.05} \\
LwF-ABC~\cite{li2017learning_LWF} & 41.85\std{0.82} & \underline{62.63}\std{0.92} & \underline{40.10}\std{0.61} 
        & \underline{24.39}\std{3.25} & \underline{53.48}\std{1.49} & \underline{25.29}\std{2.81}
        & \underline{35.45}\std{1.16} & \underline{59.63}\std{0.74} & \underline{64.32}\std{1.41} \\
\midrule
\rowcolor{gray!15}
Inf-SSM & \sota{43.82}\std{1.55} & \sota{62.82}\std{1.29} & \sota{36.34}\std{1.54} 
        & \sota{26.53}\std{3.10} & \sota{54.24}\std{1.87} & \sota{24.00}\std{1.80}
        & \sota{39.88}\std{2.43} & \sota{62.28}\std{2.33} & \sota{55.85}\std{4.05} \\
\bottomrule
\end{tabular}
\end{table*}

As shown in \cref{tab:cl_results_abc_ucl}, Inf-SSM outperforms UCL in all metric instances. On average, Inf-SSM outperforms UCL~\cite{ahn2019uncertainty} by 13.72\% in $\mathrm{AA}$, 5.00\% in $\mathrm{AIA}$, and 24.43\% in $\mathrm{FM}$.
\newpage

\section{Distance Equivalence on the Grassmannian}
\label{app:chordal_equiv}

In this section, we consider two infinite observability subspaces of an SSM, $\mathcal{S}_1$ and $\mathcal{S}_2$, with principal angles $\theta_1,\dots,\theta_n$ between them. The chordal distance~\cite{dhillon2008constructing_Chordal} is
\begin{align*}
d_{\mathrm{chord}}(\mathcal{S}_1,\mathcal{S}_2)
= \sqrt{\sum_{i=1}^n \sin^2 \theta_i }.
\end{align*}

Although the distances considered below are distinct metrics on the Grassmannian, they are locally equivalent near $\mathcal{S}_1=\mathcal{S}_2$ because they share the same second-order behavior in the principal angles. To show that the Binet--Cauchy, Fubini--Study, and Martin distances are all equivalent to the chordal distance as $\mathcal{S}_2 \to \mathcal{S}_1$, it suffices to evaluate
\begin{align*}
\lim_{\mathcal{S}_2 \to \mathcal{S}_1}
\frac{d_{\mathrm{Gr}}^2}{d_{\mathrm{chord}}^2},
\qquad
d_{\mathrm{Gr}} \in
\{d_{\mathrm{Binet}}, d_{\mathrm{Fubini}}, d_{\mathrm{Martin}}\}.
\end{align*}

Let $\theta=(\theta_1,\dots,\theta_n)$. Since the Taylor expansion of $\sin^2x$ (Maclaurin series centered at $x = 0$) is
\[
\sin^2 x = x^2 + \mathcal{O}(x^4)
\quad \text{as } x\to 0,
\]
we have
\begin{equation}
d_{\mathrm{chord}}^2
= \sum_{i=1}^n \sin^2 \theta_i
= \sum_{i=1}^n \theta_i^2 + \mathcal{O}(\|\theta\|^4).
\label{eq:taylor_chordal}
\end{equation}

\begin{lemma}[Binet--Cauchy distance and chordal distance equivalence]
\label{lem:binet_equiv}
The Binet--Cauchy distance~\cite{ye2016schubert} is
\begin{align*}
d_{\mathrm{Binet}}(\mathcal{S}_1,\mathcal{S}_2)
= \sqrt{1 - \prod_{i=1}^n \cos^2(\theta_i)},
\end{align*}
and $d_{\mathrm{Binet}}$ is equivalent to $d_{\mathrm{chord}}$ as $\mathcal{S}_2 \to \mathcal{S}_1$.
\end{lemma}

\begin{proof}
As $\mathcal{S}_2 \to \mathcal{S}_1$, we have $\theta_i \to 0$ for all $i$, so it is enough to evaluate
\begin{align*}
\lim_{\theta \to 0}
\frac{1 - \prod_{i=1}^n \cos^2(\theta_i)}
{\sum_{i=1}^n \sin^2 \theta_i}.
\end{align*}
Using
\[
\cos^2 x = 1 - x^2 + \mathcal{O}(x^4),
\]
we obtain
\begin{align*}
\prod_{i=1}^n \cos^2(\theta_i)
&=
\prod_{i=1}^n
\left(1-\theta_i^2+\mathcal{O}(\theta_i^4)\right) \\
&=
1-\sum_{i=1}^n \theta_i^2 + \mathcal{O}(\|\theta\|^4).
\end{align*}
Therefore,
\[
1-\prod_{i=1}^n \cos^2(\theta_i)
=
\sum_{i=1}^n \theta_i^2 + \mathcal{O}(\|\theta\|^4).
\]
Combining this with \cref{eq:taylor_chordal},
\begin{align*}
\lim_{\theta \to 0}
\frac{d_{\mathrm{Binet}}^2}{d_{\mathrm{chord}}^2}
=
\lim_{\theta \to 0}
\frac{\sum_{i=1}^n \theta_i^2 + \mathcal{O}(\|\theta\|^4)}
{\sum_{i=1}^n \theta_i^2 + \mathcal{O}(\|\theta\|^4)}
= 1.
\end{align*}
Hence, $d_{\mathrm{Binet}}$ and $d_{\mathrm{chord}}$ are locally equivalent.
\end{proof}

\newpage

\begin{lemma}[Fubini--Study distance and chordal distance equivalence]
\label{lem:fubini_equiv}
The Fubini--Study distance~\cite{ye2016schubert} is
\begin{align*}
d_{\mathrm{Fubini}}(\mathcal{S}_1,\mathcal{S}_2)
=
\cos^{-1}\!\left(\prod_{i=1}^n \cos(\theta_i)\right),
\end{align*}
and $d_{\mathrm{Fubini}}$ is equivalent to $d_{\mathrm{chord}}$ as $\mathcal{S}_2 \to \mathcal{S}_1$.
\end{lemma}

\begin{proof}
As $\mathcal{S}_2 \to \mathcal{S}_1$, we evaluate
\begin{align*}
\lim_{\theta \to 0}
\frac{\left[\cos^{-1}\!\left(\prod_{i=1}^n \cos(\theta_i)\right)\right]^2}
{\sum_{i=1}^n \sin^2 \theta_i}.
\end{align*}
Using the Taylor series expansion (Maclaurin series centered at $x = 0$),
\[
\cos x = 1-\frac{x^2}{2}+\mathcal{O}(x^4),
\]
we have
\begin{align*}
\prod_{i=1}^n \cos(\theta_i)
&=
\prod_{i=1}^n
\left(1-\frac{\theta_i^2}{2}+\mathcal{O}(\theta_i^4)\right) \\
&=
1-\frac{1}{2}\sum_{i=1}^n \theta_i^2 + \mathcal{O}(\|\theta\|^4).
\end{align*}
Let
\[
\delta
=
1-\prod_{i=1}^n \cos(\theta_i)
=
\frac{1}{2}\sum_{i=1}^n \theta_i^2 + \mathcal{O}(\|\theta\|^4).
\]
Since
\[
\cos^{-1}(1-\delta)
=
\sqrt{2\delta} + \mathcal{O}(\delta^{3/2}),
\]
, which is derived from the Puiseux series, it follows that
\begin{equation}
\left[\cos^{-1}\!\left(\prod_{i=1}^n \cos(\theta_i)\right)\right]^2
=
2\delta + \mathcal{O}(\delta^2)
=
\sum_{i=1}^n \theta_i^2 + \mathcal{O}(\|\theta\|^4).
\label{eq:taylor_arccos}
\end{equation}
Using \cref{eq:taylor_chordal,eq:taylor_arccos}, we conclude that
\begin{align*}
\lim_{\theta \to 0}
\frac{d_{\mathrm{Fubini}}^2}{d_{\mathrm{chord}}^2}
=
\lim_{\theta \to 0}
\frac{\sum_{i=1}^n \theta_i^2 + \mathcal{O}(\|\theta\|^4)}
{\sum_{i=1}^n \theta_i^2 + \mathcal{O}(\|\theta\|^4)}
= 1.
\end{align*}
Hence, $d_{\mathrm{Fubini}}$ and $d_{\mathrm{chord}}$ are locally equivalent.
\end{proof}

\newpage

\begin{lemma}[Martin distance and chordal distance equivalence]
\label{lem:martin_equiv}
The Martin distance~\cite{martin2000metric,de2002subspace} is
\begin{align*}
d_{\mathrm{Martin}}(\mathcal{S}_1,\mathcal{S}_2)
=
\sqrt{-\log \prod_{i=1}^n \cos^2 \theta_i},
\end{align*}
and $d_{\mathrm{Martin}}$ is equivalent to $d_{\mathrm{chord}}$ as $\mathcal{S}_2 \to \mathcal{S}_1$.
\end{lemma}

\begin{proof}
As $\mathcal{S}_2 \to \mathcal{S}_1$, we evaluate
\begin{align*}
\lim_{\theta \to 0}
\frac{-\log \prod_{i=1}^n \cos^2 \theta_i}
{\sum_{i=1}^n \sin^2 \theta_i}.
\end{align*}
Using
\[
\log \prod_{i=1}^n \cos^2 \theta_i
=
\sum_{i=1}^n \log(\cos^2 \theta_i),
\]
together with Taylor series expansion of $\cos^2 x$ (Maclaurin series at $x=0$),
\[
\log(\cos^2 x) = -x^2 + \mathcal{O}(x^4)
\quad \text{as } x\to 0,
\]
we obtain
\begin{equation}
-\log \prod_{i=1}^n \cos^2 \theta_i
=
\sum_{i=1}^n \theta_i^2 + \mathcal{O}(\|\theta\|^4).
\label{eq:taylor_martin}
\end{equation}
Combining \cref{eq:taylor_martin,eq:taylor_chordal}, we get
\begin{align*}
\lim_{\theta \to 0}
\frac{d_{\mathrm{Martin}}^2}{d_{\mathrm{chord}}^2}
=
\lim_{\theta \to 0}
\frac{\sum_{i=1}^n \theta_i^2 + \mathcal{O}(\|\theta\|^4)}
{\sum_{i=1}^n \theta_i^2 + \mathcal{O}(\|\theta\|^4)}
= 1.
\end{align*}
Hence, $d_{\mathrm{Martin}}$ and $d_{\mathrm{chord}}$ are locally equivalent.
\end{proof}
\newpage
\section{Additional implementation details}

\subsection{Baselines} 
\label{app:baselines}
Our objective is to evaluate Inf-SSM under a continual learning (CL) protocol that does not rely on CNN- or Transformer-specific architectural assumptions, so that both Inf-SSM and all baselines can be instantiated fairly on the Vim backbone. Accordingly, we require baselines that:
\begin{enumerate}
    \item Represent the main CL paradigms used in CIL or EFCIL.
    \item Can be adapted to the Vim backbone without ad hoc, architecture-specific redesign.
\end{enumerate}
We therefore adopt the following baselines, all implemented on Vim-Small under a shared training protocol (datasets, task splits, and metrics are described in \textsection\ref{sec:experiments}, \textsection\ref{app:cl_eval_metric}, and \textsection\ref{app:add_implement}).

\noindent For replay-based methods:
\begin{itemize}
    \item \textbf{ER}~\cite{riemer2018learning}is a canonical replay-based baseline. It measures how much Inf-SSM can further reduce forgetting when explicit rehearsal is allowed.
    \item \textbf{LUCIR}~\cite{hou2019learning} and \textbf{X-DER}~\cite{boschini2022class} are strong hybrid methods that combine replay with regularization or contrastive mechanisms. They are widely used in CIL evaluations and indicate whether Inf-SSM still brings gains on top of competitive replay-regularization pipelines.
    \item \textbf{L2P-R}~\cite{wang2022learning} is a prompt-based method adapted to Vim-Small (see \textsection\ref{app:l2p_ssm}) to test compatibility of Inf-SSM with token-level adaptation strategies in SSM architectures.
    \item \textbf{CLFD}~\cite{liu2024continual} is a frequency-domain method representing the latest CL designs. We include it to show Inf-SSM’s benefit even when features are transformed into alternative domains.
\end{itemize}

\noindent For EFCIL methods:
\begin{itemize}
    \item \textbf{EWC}~\cite{kirkpatrick2017EWC} serves as a canonical example of sensitivity-based regularization. 
    \item \textbf{SI}~\cite{zenke2017SI} serves as a baseline for the synaptic-level importance approach. 
    \item \textbf{MAS}~\cite{aljundi2018MAS} offers comparison against methods based on Hebbian learning theory.
    \item \textbf{LwF}~\cite{li2017learning_LWF} is adapted to distill SSM state using the Frobenius norm applied explicitly on $(\mA, \mC)$ for LwF-AC and $(\mA, \mB, \mC)$ for LwF-ABC. LwF serves as a cornerstone distillation-based method that explicitly regularizes SSM states via the Frobenius norm and provides strong evidence for the importance of a $\mathbf{P}$-equivalence-aware distance measure in SSMs.
\end{itemize}

\noindent Together, these baselines span a wide range of CL families and are heavily used in prior CIL or EFCIL literature. We intentionally exclude complex hybrid or prototype-based EFCIL methods in our isolation test because they either violate the EFCIL assumptions (e.g., by storing prototypes or intermediate features) or require heavy, architecture-specific modifications that are not directly compatible with a clean Vim-SSM instantiation. More in-depth discussions on several such methods are included in \textsection\ref{app:related_work_add}.

Regarding Mamba-based CL methods like MambaCL~\cite{ZhaoMAMBACL}, Mamba-CL~\cite{cheng2024mamba}, and Mamba-FSCIL~\cite{LiMAMBAFscil} (see \textsection\ref{app:mamba_cl} for details), these works focus on scenario-specific goals (e.g., online CL, few-shot class-incremental learning, task-conditional adaptation) and do not exploit the rich geometrical structures of SSMs. We thus view them as future integration targets that are orthogonal to our research question, and not direct baselines for validating our core claim. 

\noindent\textbf{Summary.} Our baseline set is chosen to cover prominent CL methods under a unified SSM backbone and evaluation protocol, while avoiding methods whose assumptions (stored features, heavily modified architectures, or different CL scenarios) are misaligned with the problem setting and goals of our work. Under these representative baselines, Inf-SSM consistently reduces forgetting and improves accuracy, supporting our claim that geometry-aware observability regularization is an effective and broadly compatible CL regularization algorithm.

\newpage
\subsection{Hyperparameters and Compute}
\label{app:add_implement}

For all experiments, we run on seeds 0, 10, and 100 with the same set of hyperparameters as shown below.  All RGB images are resized to $224\times224$ before training and evaluations. For all datasets, we split each into 5 and 10 sequential tasks, where each task has an equal number of classes sampled from the corresponding datasets. For backbone, we utilized Vim-small~\cite{zhu2024VIM} and kept all the hyperparameters in Vim as default unless mentioned otherwise.

The experiments are conducted on various machines available, which are A5500 GPU, A40 GPU, A100 GPU, and H100 GPU. All experiments are conducted on a single GPU only without distributed training.
\begin{table}[ht]
    \centering
    \caption{Hyperparameters for VIM-Small model on 5-task continual learning benchmark. All regularization coefficients ($\lambda$) are shown in base units. Learning rates (LR) follow cosine decay schedules.}
    \label{tab:vim_hyperparams_5task}
    \begin{tabular}{lccc}
        \toprule
        \textbf{Hyperparameter} & \textbf{ImageNet-R} & \textbf{CIFAR-100} & \textbf{Caltech-256} \\
        \midrule
        Batch Size & 128 & 128 & 128 \\
        Training Epochs & 40 & 40 & 40 \\
        Base LR & $5.00\times10^{-4}$ & $1.00\times10^{-5}$ & $1.00\times10^{-4}$ \\
        Warmup LR & $1.00\times10^{-4}$ & $1.00\times10^{-6}$ & $1.00\times10^{-4}$ \\
        Minimum LR & $1.00\times10^{-5}$ & $1.00\times10^{-7}$ & $1.00\times10^{-5}$ \\
        Task LR Scaling & $2.50\times10^{-1}$ & $5.00\times10^{-1}$ & $5.00\times10^{-1}$ \\
        Weight Decay & $1.00\times10^{-1}$ & $1.00\times10^{-1}$ & $1.00\times10^{-1}$ \\
        \midrule
        EWC-E-$\lambda$ & $2.00\times10^{2}$ & $2.50\times10^{3}$ & $1.00\times10^{4}$ \\
        EWC-$\gamma$ & $7.50\times10^{-1}$ & $7.50\times10^{-1}$ & $7.50\times10^{-1}$ \\
        SI-$C$ & $1.00\times10^{3}$ & $1.00\times10^{5}$ & $5.00\times10^{4}$ \\
        SI-$\xi$ & $9.00\times10^{-1}$ & $9.00\times10^{-1}$ & $9.00\times10^{-1}$ \\
        MAS-$\lambda$ & $1.00\times10^{1}$ & $1.00\times10^{1}$ & $1.00\times10^{2}$ \\
        MSE-$\gamma$ & $1.00\times10^{2}$ & $5.00\times10^{1}$ & $1.00\times10^{2}$ \\
        MSE-$\lambda$ & $5.00\times10^{2}$ & $2.50\times10^{2}$ & $5.00\times10^{2}$ \\
        \midrule
        Inf-SSM-$\lambda$ & $1.00\times10^{6}$ & $2.00\times10^{5}$ & $2.50\times10^{6}$ \\
        Inf-SSM+ -$\lambda$ & $1.00\times10^{6}$ & $2.00\times10^{3}$ & $2.50\times10^{6}$ \\
        Inf-SSM+ -$\gamma$ & $1.00\times10^{0}$ & $1.00\times10^{2}$ & $1.00\times10^{2}$ \\
        \bottomrule
    \end{tabular}
\end{table}

\begin{table}[ht]
    \centering
    \caption{Hyperparameters for VIM-Small model on 10-task continual learning benchmark. Configuration follows same conventions as Table~\ref{tab:vim_hyperparams_5task}.}
    \label{tab:vim_hyperparams_10task}
    \begin{tabular}{lccc}
        \toprule
        \textbf{Hyperparameter} & \textbf{ImageNet-R} & \textbf{CIFAR-100} & \textbf{Caltech-256} \\
        \midrule
        Batch Size & 128 & 128 & 128 \\
        Training Epochs & 40 & 40 & 40 \\
        Base LR & $5.00\times10^{-4}$ & $1.00\times10^{-5}$ & $1.00\times10^{-4}$ \\
        Warmup LR & $1.00\times10^{-4}$ & $1.00\times10^{-6}$ & $1.00\times10^{-5}$ \\
        Minimum LR & $1.00\times10^{-5}$ & $1.00\times10^{-7}$ & $1.00\times10^{-5}$ \\
        Task LR Scaling & $2.50\times10^{-1}$ & $5.00\times10^{-1}$ & $5.00\times10^{-1}$ \\
        Weight Decay & $1.00\times10^{-1}$ & $1.00\times10^{-1}$ & $1.00\times10^{-1}$ \\
        \midrule
        EWC-$\lambda$ & $5.00\times10^{2}$ & $1.00\times10^{4}$ & $1.00\times10^{3}$ \\
        EWC-$\gamma$ & $7.50\times10^{-1}$ & $7.50\times10^{-1}$ & $7.50\times10^{-1}$ \\
        SI-$c$ & $5.00\times10^{4}$ & $5.00\times10^{4}$ & $5.00\times10^{4}$ \\
        SI-$\xi$ & $9.00\times10^{-1}$ & $8.00\times10^{-1}$ & $9.00\times10^{-1}$ \\
        MAS-$\lambda$ & $1.00\times10^{2}$ & $1.00\times10^{1}$ & $1.00\times10^{2}$ \\
        MSE-$\gamma$ & $1.00\times10^{2}$ & $1.00\times10^{1}$ & $5.00\times10^{2}$ \\
        MSE-$\lambda$ & $5.00\times10^{2}$ & $5.00\times10^{1}$ & $5.00\times10^{2}$ \\
        \midrule
        Inf-SSM-$\lambda$ & $2.50\times10^{5}$ & $3.00\times10^{4}$ & $2.50\times10^{6}$ \\
        Inf-SSM+ -$\lambda$ & $1.50\times10^{5}$ & $2.00\times10^{2}$ & $1.00\times10^{4}$ \\
        Inf-SSM+ -$\gamma$ & $1.00\times10^{2}$ & $1.00\times10^{1}$ & $1.00\times10^{1}$ \\
        \bottomrule
    \end{tabular}
\end{table}

\newpage

\begin{table}[ht]
    \centering
    \caption{Hyperparameters for ER, LUCIR, X-DER, L2P-R, and CLFD methods integration tests with Inf-SSM for ImageNet-R 5 tasks setting. Configuration includes both general training settings and method-specific parameters.}
    \label{tab:buf_hyperparams_methods_imnet_5tasks}
    \begin{tabular}{lccccc}
        \toprule
        \textbf{Hyperparameter} & \textbf{ER} & \textbf{LUCIR} & \textbf{X-DER} & \textbf{L2P-R}  & \textbf{CLFD} \\
        \midrule
        Batch Size & 128 & 128 & 32 & 32 & 64 \\
        Training Epochs & 40 & 40 & 20 & 20 & 20  \\
        Base LR & $5.00\times10^{-4}$ & $5.00\times10^{-4}$ & $5.00\times10^{-4}$ & $5.00\times10^{-4}$ & $5.00\times10^{-4}$ \\
        Warmup LR & $1.00\times10^{-4}$ & $1.00\times10^{-4}$ & $1.00\times10^{-4}$  & $1.00\times10^{-4}$ & $1.00\times10^{-4}$\\
        Minimum LR & $1.00\times10^{-4}$ & $1.00\times10^{-4}$ & $1.00\times10^{-4}$ & $1.00\times10^{-5}$  & $1.00\times10^{-5}$  \\
        Task LR Scaling & $2.50\times10^{-1}$ & $2.50\times10^{-1}$ & $2.50\times10^{-1}$ & $2.50\times10^{-1}$  & $2.50\times10^{-1}$ \\
        Weight Decay & $1.00\times10^{-1}$ & $1.00\times10^{-1}$ & $1.00\times10^{-1}$ & $1.00\times10^{-1}$ & $1.00\times10^{-1}$\\
        Buffer Size & $5.00\times10^{2}$ & $5.00\times10^{2}$ & $1.00\times10^{3}$ & $1.00\times10^{3}$  & $1.00\times10^{3}$ \\
        \midrule
        LUCIR-$\lambda_{\text{base}}$ & -- & $5.00\times10^{-1}$ & -- & -- & --\\
        LUCIR-$\lambda_{\text{MR}}$ & -- & $1.00\times10^{-1}$ & -- & -- & -- \\
        LUCIR-$K_{\text{MR}}$ & -- & $2.00\times10^{0}$ & -- & -- & -- \\
        LUCIR-MR Margin & -- & $5.00\times10^{-2}$ & -- & -- & -- \\
        \midrule
        X-DER-$\gamma$ & -- & -- & $8.50\times10^{-1}$ & -- & -- \\
        X-DER-Temp & -- & -- & $7.00\times10^{-2}$  & -- & -- \\
        X-DER-Base Temp & -- & -- & $7.00\times10^{-2}$ & -- & -- \\
        X-DER-$\alpha$ & -- & -- & $3.00\times10^{-1}$ & -- & -- \\
        X-DER-$\beta$ & -- & -- & $1.80\times10^{0}$ & -- & -- \\
        X-DER-SimCLR Batch Size & -- & -- & $3.20\times10^{1}$ & -- & -- \\
        X-DER-SimCLR Num Augs & -- & -- & $2.00\times10^{0}$ & -- & -- \\
        X-DER-$\lambda$ & -- & -- & $5.00\times10^{-2}$ & -- & -- \\
        X-DER-dp Weight & -- & -- & $1.00\times10^{-1}$ & -- & -- \\
        X-DER-Contr Margin & -- & -- & $3.00\times10^{-1}$ & -- & -- \\
        X-DER-Constr $\eta$ & -- & -- & $1.00\times10^{-1}$ & -- & -- \\
        X-DER-Future Constr & -- & -- & $1.00\times10^{0}$ & -- & -- \\
        X-DER-Part Constr & -- & -- & $0.00\times10^{0}$ & -- & -- \\
        \midrule
        L2P-R-Pull-Constr & -- & -- & -- & $1.00\times10^{-1}$ & --  \\
        \midrule
        Inf-SSM-$\lambda$ & $1.50\times10^{5}$ & $5.00\times10^{2}$ & $5.00\times10^{2}$ & $5.00\times10^{3}$ & $5.00\times10^{3}$ \\
        \bottomrule
    \end{tabular}
\end{table}

\begin{table}[ht]
    \centering
    \caption{Hyperparameters for ER, LUCIR, X-DER, L2P-R, and CLFD methods integration tests with Inf-SSM for ImageNet-R 10 tasks setting. Configuration includes both general training settings and method-specific parameters.}
    \label{tab:buf_hyperparams_methods_imnet_10tasks}
    \begin{tabular}{lccccc}
        \toprule
        \textbf{Hyperparameter} & \textbf{ER} & \textbf{LUCIR} & \textbf{X-DER} & \textbf{L2P-R}  & \textbf{CLFD} \\ \\
        \midrule
        Batch Size & 128 & 128 & 32 & 32 & 64 \\
        Training Epochs & 40 & 40 & 20 & 20 & 20  \\
        Base LR & $5.00\times10^{-4}$ & $5.00\times10^{-4}$ & $5.00\times10^{-4}$  &  $1.00\times10^{-4}$ &  $1.00\times10^{-4}$\\
        Warmup LR & $1.00\times10^{-4}$ & $1.00\times10^{-4}$ & $1.00\times10^{-4}$ & $1.00\times10^{-5}$  &  $1.00\times10^{-4}$\\
        Minimum LR & $1.00\times10^{-4}$ & $1.00\times10^{-4}$ & $1.00\times10^{-4}$ & $1.00\times10^{-5}$ & $1.00\times10^{-5}$\\
        Task LR Scaling & $2.50\times10^{-1}$ & $2.50\times10^{-1}$ & $2.50\times10^{-1}$ & $5.00\times10^{-1}$ & $2.50\times10^{-1}$ \\
        Weight Decay & $1.00\times10^{-1}$ & $1.00\times10^{-1}$ & $1.00\times10^{-1}$  & $1.00\times10^{-1}$ & $1.00\times10^{-1}$ \\
        Buffer Size & $5.00\times10^{2}$ & $5.00\times10^{2}$ & $1.00\times10^{3}$ & $1.00\times10^{3}$ & $1.00\times10^{3}$\\
        \midrule
        LUCIR-$\lambda_{\text{base}}$ & -- & $5.00\times10^{-1}$ & -- & -- & -- \\
        LUCIR-$\lambda_{\text{MR}}$ & -- & $1.00\times10^{-1}$ & -- & -- & -- \\
        LUCIR-$K_{\text{MR}}$ & -- & $2.00\times10^{0}$ & -- & -- & -- \\
        LUCIR-MR Margin & -- & $5.00\times10^{-2}$ & -- & -- & -- \\
        \midrule
        X-DER-$\gamma$ & -- & -- & $8.50\times10^{-1}$ & -- & -- \\
        X-DER-Temp & -- & -- & $7.00\times10^{-2}$ & -- & -- \\
        X-DER-Base Temp & -- & -- & $7.00\times10^{-2}$ & -- & -- \\
        X-DER-$\alpha$ & -- & -- & $3.00\times10^{-1}$ & -- & -- \\
        X-DER-$\beta$ & -- & -- & $1.80\times10^{0}$ & -- & -- \\
        X-DER-SimCLR Batch Size & -- & -- & $3.20\times10^{1}$ & -- & -- \\
        X-DER-SimCLR Num Augs & -- & -- & $2.00\times10^{0}$ & -- & -- \\
        X-DER-$\lambda$ & -- & -- & $5.00\times10^{-2}$ & -- & -- \\
        X-DER-dp Weight & -- & -- & $1.00\times10^{-1}$ & -- & -- \\
        X-DER-Contr Margin & -- & -- & $3.00\times10^{-1}$ & -- & -- \\
        X-DER-Constr $\eta$ & -- & -- & $1.00\times10^{-1}$ & -- & --  \\
        X-DER-Future Constr & -- & -- & $1.00\times10^{0}$ & -- & -- \\
        X-DER-Part Constr & -- & -- & $0.00\times10^{0}$ & -- & -- \\
        \midrule
        L2P-R-Pull-Constr & -- & -- & -- & $1.00\times10^{-1}$ & --\\
        \midrule
        Inf-SSM-$\lambda$ & $1.50\times10^{5}$ & $1.00\times10^{3}$ & $5.00\times10^{2}$ & $5.00\times10^{4}$ & $1.00\times10^{4}$\\
        \bottomrule
    \end{tabular}
\end{table}

\begin{table}[ht]
    \centering
    \caption{Hyperparameters for ER, LUCIR, X-DER, L2P-R, and CLFD methods integration tests with Inf-SSM for CIFAR-100 5 tasks setting. Configuration includes both general training settings and method-specific parameters.}
    \label{tab:buf_hyperparams_methods_cifar_5tasks}
    \begin{tabular}{lccccc}
        \toprule
        \textbf{Hyperparameter} & \textbf{ER} & \textbf{LUCIR} & \textbf{X-DER} & \textbf{L2P-R}  & \textbf{CLFD} \\
        \midrule
        Batch Size & 128 & 64 & 32 & 32 & 64 \\
        Training Epochs & 40 & 40 & 20 & 20 & 20 \\
        Base LR & $5.00\times10^{-4}$ & $5.00\times10^{-4}$ & $5.00\times10^{-4}$ & $5.00\times10^{-4}$ & $1.00\times10^{-4}$  \\
        Warmup LR & $1.00\times10^{-4}$ & $1.00\times10^{-4}$ & $1.00\times10^{-4}$ & $1.00\times10^{-4}$ &$5.00\times10^{-5}$ \\
        Minimum LR & $1.00\times10^{-5}$ & $1.00\times10^{-5}$ & $1.00\times10^{-5}$  & $1.00\times10^{-5}$ & $1.00\times10^{-5}$ \\
        Task LR Scaling & $2.50\times10^{-1}$ & $2.50\times10^{-1}$ & $2.50\times10^{-1}$ & $2.50\times10^{-1}$ & $5.00\times10^{-1}$ \\
        Weight Decay & $1.00\times10^{-1}$ & $1.00\times10^{-1}$ & $1.00\times10^{-1}$ & $1.00\times10^{-1}$ & $1.00\times10^{-1}$ \\
        Buffer Size & $1.00\times10^{3}$ & $1.00\times10^{3}$ & $1.00\times10^{3}$ & $2.00\times10^{3}$ & $1.00\times10^{3}$\\
        \midrule
        LUCIR-$\lambda_{\text{base}}$ & -- & $5.00\times10^{-1}$ & --  & --  & -- \\
        LUCIR-$\lambda_{\text{MR}}$ & -- & $1.00\times10^{-1}$ & -- & --  & -- \\
        LUCIR-$K_{\text{MR}}$ & -- & $2.00\times10^{0}$ & -- & --  & -- \\
        LUCIR-MR Margin & -- & $5.00\times10^{-2}$ & -- & --  & -- \\
        \midrule
        X-DER-$\gamma$ & -- & -- & $8.50\times10^{-1}$ & --  & -- \\
        X-DER-Temp & -- & -- & $7.00\times10^{-2}$& --  & -- \\
        X-DER-Base Temp & -- & -- & $7.00\times10^{-2}$ & --  & --\\
        X-DER-$\alpha$ & -- & -- & $3.00\times10^{-1}$& --  & -- \\
        X-DER-$\beta$ & -- & -- & $1.80\times10^{0}$ & --  & --\\
        X-DER-SimCLR Batch Size & -- & -- & $3.20\times10^{1}$ & --  & --\\
        X-DER-SimCLR Num Augs & -- & -- & $2.00\times10^{0}$ & --  & --\\
        X-DER-$\lambda$ & -- & -- & $5.00\times10^{-2}$& --  & -- \\
        X-DER-dp Weight & -- & -- & $1.00\times10^{-1}$& --  & -- \\
        X-DER-Contr Margin & -- & -- & $3.00\times10^{-1}$ & --  & --\\
        X-DER-Constr $\eta$ & -- & -- & $1.00\times10^{-1}$ & --  & --\\
        X-DER-Future Constr & -- & -- & $1.00\times10^{0}$ & --  & --\\
        X-DER-Part Constr & -- & -- & $0.00\times10^{0}$ & --  & --\\
        \midrule
        L2P-R-Pull-Constr & -- & -- & -- & $1.00\times10^{-1}$ & --\\
        \midrule
        Inf-SSM-$\lambda$ & $1.00\times10^{3}$ & $1.00\times10^{3}$ & $5.00\times10^{2}$ & $1.00\times10^{4}$ & $1.00\times10^{4}$ \\
        \bottomrule
    \end{tabular}
\end{table}

\begin{table}[ht]
    \centering
    \caption{Hyperparameters for ER, LUCIR, X-DER, L2P-R, and CLFD methods integration tests with Inf-SSM for CIFAR-100 10 tasks setting. Configuration includes both general training settings and method-specific parameters.}
    \label{tab:buf_hyperparams_methods_cifar_10tasks}
    \begin{tabular}{lccccc}
        \toprule
        \textbf{Hyperparameter} & \textbf{ER} & \textbf{LUCIR} & \textbf{X-DER} & \textbf{L2P-R}  & \textbf{CLFD} \\ \\
        \midrule
        Batch Size & 128 & 64 & 32 & 32 & 64 \\
        Training Epochs & 40 & 40 & 20 & 40 & 20 \\
        Base LR & $5.00\times10^{-4}$ & $5.00\times10^{-4}$ & $5.00\times10^{-4}$ & $5.00\times10^{-4}$ & $1.00\times10^{-4}$ \\
        Warmup LR & $1.00\times10^{-4}$ & $1.00\times10^{-4}$ & $1.00\times10^{-4}$ & $1.00\times10^{-4}$ & $5.00\times10^{-5}$\\
        Minimum LR & $1.00\times10^{-5}$ & $1.00\times10^{-5}$ & $1.00\times10^{-5}$ & $1.00\times10^{-5}$ & $1.00\times10^{-5}$ \\
        Task LR Scaling & $2.50\times10^{-1}$ & $2.50\times10^{-1}$ & $2.50\times10^{-1}$ & $2.50\times10^{-1}$ & $5.00\times10^{-1}$\\
        Weight Decay & $1.00\times10^{-1}$ & $1.00\times10^{-1}$ & $1.00\times10^{-1}$ & $1.00\times10^{-1}$ & $1.00\times10^{-1}$ \\
        Buffer Size & $1.00\times10^{3}$ & $1.00\times10^{3}$ & $1.00\times10^{3}$ & $2.00\times10^{3}$ & $1.00\times10^{3}$\\
        \midrule
        LUCIR-$\lambda_{\text{base}}$ & -- & $5.00\times10^{-1}$ & -- & -- & -- \\
        LUCIR-$\lambda_{\text{MR}}$ & -- & $1.00\times10^{-1}$ & --& -- & -- \\
        LUCIR-$K_{\text{MR}}$ & -- & $2.00\times10^{0}$ & -- & -- & --\\
        LUCIR-MR Margin & -- & $5.00\times10^{-2}$ & -- & -- & --\\
        \midrule
        X-DER-$\gamma$ & -- & -- & $8.50\times10^{-1}$ & -- & -- \\
        X-DER-Temp & -- & -- & $7.00\times10^{-2}$ & -- & -- \\
        X-DER-Base Temp & -- & -- & $7.00\times10^{-2}$ & -- & -- \\
        X-DER-$\alpha$ & -- & -- & $3.00\times10^{-1}$ & -- & -- \\
        X-DER-$\beta$ & -- & -- & $1.80\times10^{0}$ & -- & -- \\
        X-DER-SimCLR Batch Size & -- & -- & $3.20\times10^{1}$ & -- & --\\
        X-DER-SimCLR Num Augs & -- & -- & $2.00\times10^{0}$ & -- & --\\
        X-DER-$\lambda$ & -- & -- & $5.00\times10^{-2}$ & -- & --\\
        X-DER-dp Weight & -- & -- & $1.00\times10^{-1}$ & -- & --\\
        X-DER-Contr Margin & -- & -- & $3.00\times10^{-1}$ & -- & --\\
        X-DER-Constr $\eta$ & -- & -- & $1.00\times10^{-1}$& -- & -- \\
        X-DER-Future Constr & -- & -- & $1.00\times10^{0}$& -- & -- \\
        X-DER-Part Constr & -- & -- & $0.00\times10^{0}$& -- & -- \\
        \midrule
        L2P-R-Pull-Constr & -- & -- & -- & $1.00\times10^{-1}$ & --\\
        \midrule
        Inf-SSM-$\lambda$ & $2.50\times10^{2}$ & $1.00\times10^{0}$ & $1.00\times10^{3}$ & $2.00\times10^{3}$ & $1.00\times10^{4}$ \\
        \bottomrule
    \end{tabular}
\end{table}

\begin{table}[ht]
    \centering
    \caption{Hyperparameters for ER, LUCIR, X-DER, L2P-R, and CLFD methods integration tests with Inf-SSM for Caltech-256 5 tasks setting. Configuration includes both general training settings and method-specific parameters.}
    \label{tab:buf_hyperparams_methods_caltech_5tasks}
    \begin{tabular}{lccccc}
        \toprule
        \textbf{Hyperparameter} & \textbf{ER} & \textbf{LUCIR} & \textbf{X-DER} & \textbf{L2P-R}  & \textbf{CLFD} \\
        \midrule
        Batch Size & 128 & 128 & 32 & 32 & 64\\
        Training Epochs & 40 & 40 & 20 & 40 & 20\\
        Base LR & $1.00\times10^{-4}$ & $1.00\times10^{-4}$ & $1.00\times10^{-4}$ & $1.00\times10^{-4}$ & $1.00\times10^{-4}$ \\
        Warmup LR & $1.00\times10^{-5}$ & $1.00\times10^{-5}$ & $1.00\times10^{-5}$ & $1.00\times10^{-5}$ & $1.00\times10^{-5}$ \\
        Minimum LR & $1.00\times10^{-5}$ & $1.00\times10^{-5}$ & $1.00\times10^{-5}$ & $1.00\times10^{-5}$ & $1.00\times10^{-5}$ \\
        Task LR Scaling & $5.00\times10^{-1}$ & $5.00\times10^{-1}$ & $5.00\times10^{-1}$ & $5.00\times10^{-1}$ & $5.00\times10^{-1}$ \\
        Weight Decay & $1.00\times10^{-1}$ & $1.00\times10^{-1}$ & $1.00\times10^{-1}$ & $1.00\times10^{-1}$ & $1.00\times10^{-1}$ \\
        Buffer Size & $1.00\times10^{3}$ & $1.00\times10^{3}$ & $1.00\times10^{3}$ & $1.00\times10^{3}$ & $1.00\times10^{3}$ \\
        \midrule
        LUCIR-$\lambda_{\text{base}}$ & -- & $5.00\times10^{-1}$ & -- & -- & -- \\
        LUCIR-$\lambda_{\text{MR}}$ & -- & $1.00\times10^{-1}$ & -- & -- & -- \\
        LUCIR-$K_{\text{MR}}$ & -- & $2.00\times10^{0}$ & -- & -- & -- \\
        LUCIR-MR Margin & -- & $5.00\times10^{-2}$ & -- & -- & -- \\
        \midrule
        X-DER-$\gamma$ & -- & -- & $8.50\times10^{-1}$ & -- & -- \\
        X-DER-Temp & -- & -- & $7.00\times10^{-2}$ & -- & --\\
        X-DER-Base Temp & -- & -- & $7.00\times10^{-2}$ & -- & -- \\
        X-DER-$\alpha$ & -- & -- & $3.00\times10^{-1}$ & -- & --\\
        X-DER-$\beta$ & -- & -- & $1.80\times10^{0}$  & -- & --\\
        X-DER-SimCLR Batch Size & -- & -- & $3.20\times10^{1}$  & -- & --\\
        X-DER-SimCLR Num Augs & -- & -- & $2.00\times10^{0}$ & -- & --\\
        X-DER-$\lambda$ & -- & -- & $5.00\times10^{-2}$ & -- & --\\
        X-DER-dp Weight & -- & -- & $1.00\times10^{-1}$ & -- & --\\
        X-DER-Contr Margin & -- & -- & $3.00\times10^{-1}$& -- & -- \\
        X-DER-Constr $\eta$ & -- & -- & $1.00\times10^{-1}$ & -- & --\\
        X-DER-Future Constr & -- & -- & $1.00\times10^{0}$ & -- & --\\
        X-DER-Part Constr & -- & -- & $0.00\times10^{0}$ & -- & --\\
        \midrule
        L2P-R-Pull-Constr & -- & -- & -- & $1.00\times10^{-1}$ & --\\
        \midrule
        Inf-SSM-$\lambda$ & $1.00\times10^{3}$ & $3.00\times10^{3}$ & $5.00\times10^{2}$ & $5.00\times10^{2}$ & $1.00\times10^{3}$ \\
        \bottomrule
    \end{tabular}
\end{table}

\begin{table}[ht]
    \centering
    \caption{Hyperparameters for ER, LUCIR, X-DER, L2P-R, and CLFD methods integration tests with Inf-SSM for Caltech-256 10 tasks setting. Configuration includes both general training settings and method-specific parameters.}
    \label{tab:buf_hyperparams_methods_caltech_10tasks}
    \begin{tabular}{lccccc}
        \toprule
        \textbf{Hyperparameter} & \textbf{ER} & \textbf{LUCIR} & \textbf{X-DER} & \textbf{L2P-R}  & \textbf{CLFD} \\ \\
        \midrule
        Batch Size & 128 & 128 & 32 & 32 & 64\\
        Training Epochs & 40 & 40 & 20 & 40 & 20 \\
        Base LR & $1.00\times10^{-4}$ & $1.00\times10^{-4}$ & $1.00\times10^{-4}$ & $1.00\times10^{-4}$ & $1.00\times10^{-4}$ \\
        Warmup LR & $1.00\times10^{-5}$ & $1.00\times10^{-5}$& $1.00\times10^{-5}$ & $1.00\times10^{-5}$ & $1.00\times10^{-5}$ \\
        Minimum LR & $1.00\times10^{-5}$ & $1.00\times10^{-5}$ & $1.00\times10^{-5}$ & $1.00\times10^{-5}$ & $1.00\times10^{-5}$ \\
        Task LR Scaling & $5.00\times10^{-1}$ & $5.00\times10^{-1}$ & $5.00\times10^{-1}$ & $5.00\times10^{-1}$ & $5.00\times10^{-1}$ \\
        Weight Decay & $1.00\times10^{-1}$ & $1.00\times10^{-1}$ & $1.00\times10^{-1}$ & $1.00\times10^{-1}$ & $1.00\times10^{-1}$  \\
        Buffer Size & $1.00\times10^{3}$  & $1.00\times10^{3}$ & $1.00\times10^{3}$ & $1.00\times10^{3}$  & $1.00\times10^{3}$  \\
        \midrule
        LUCIR-$\lambda_{\text{base}}$ & -- & $5.00\times10^{-1}$ & -- & -- & -- \\
        LUCIR-$\lambda_{\text{MR}}$ & -- & $1.00\times10^{-1}$ & -- & -- & -- \\
        LUCIR-$K_{\text{MR}}$ & -- & $2.00\times10^{0}$ & -- & -- & -- \\
        LUCIR-MR Margin & -- & $5.00\times10^{-2}$ & -- & -- & -- \\
        \midrule
        X-DER-$\gamma$ & -- & -- & $8.50\times10^{-1}$ & -- & -- \\
        X-DER-Temp & -- & -- & $7.00\times10^{-2}$ & -- & -- \\
        X-DER-Base Temp & -- & -- & $7.00\times10^{-2}$& -- & --  \\
        X-DER-$\alpha$ & -- & -- & $3.00\times10^{-1}$& -- & -- \\
        X-DER-$\beta$ & -- & -- & $1.80\times10^{0}$ & -- & -- \\
        X-DER-SimCLR Batch Size & -- & -- & $3.20\times10^{1}$ & -- & -- \\
        X-DER-SimCLR Num Augs & -- & -- & $2.00\times10^{0}$ & -- & -- \\
        X-DER-$\lambda$ & -- & -- & $5.00\times10^{-2}$ & -- & -- \\
        X-DER-dp Weight & -- & -- & $1.00\times10^{-1}$ & -- & -- \\
        X-DER-Contr Margin & -- & -- & $3.00\times10^{-1}$ & -- & -- \\
        X-DER-Constr $\eta$ & -- & -- & $1.00\times10^{-1}$ & -- & -- \\
        X-DER-Future Constr & -- & -- & $1.00\times10^{0}$ & -- & -- \\
        X-DER-Part Constr & -- & -- & $0.00\times10^{0}$  & -- & -- \\
        \midrule
        L2P-R-Pull-Constr & -- & -- & -- & $1.00\times10^{-1}$ & --\\
        \midrule
        Inf-SSM-$\lambda$ & $1.00\times10^{3}$ & $1.00\times10^{3}$ & $5.00\times10^{2}$ & $2.00\times10^{3}$ & $5.00\times10^{3}$  \\
        \bottomrule
    \end{tabular}
\end{table}

\subsection{L2P in SSM}
\label{app:l2p_ssm}
Adapting prompt-based methods in SSM is not straightforward. This is due to the recurrent nature of SSMs, leading to prompt tokens positioning being sensitive, unlike in attention layers. In our L2P-R implementation, we insert the prompt tokens by concatenating them at the start of the patch sequence. For classification, we have frozen the [CLS] token and instead utilize the prompt tokens for downstream classification by the final linear layer.

\end{document}